%% file: master.tex
\documentclass[review, authoryear]{elsarticle}

\usepackage{hyperref}
\hypersetup{
    colorlinks=true
}
\usepackage{ulem}
\usepackage{amsmath}
\usepackage{amsthm}
\usepackage{amssymb}
\usepackage{bbold}
\usepackage{bm}
\newcommand{\R}{\mathbb{R}}

\usepackage[table, dvipsnames]{xcolor}

\definecolor{cyan}{rgb}{0,0.5,.2}
\definecolor{dcyan}{rgb}{0,0.3,.3}

 \definecolor{pink}{HTML}{ff00d4}
 \definecolor{dpink}{HTML}{99007f}

\usepackage{algorithm, algpseudocode}

\definecolor{darkgray}{HTML}{6e6e6e}

\algnewcommand\algorithmicto{\textbf{to}}
\algnewcommand\RETURN{\State \textbf{return} }

\journal{Journal of Econometrics (status: accepted)}

\usepackage{subcaption}

\newcommand{\E}{{\mathbb E}}

\renewcommand{\R}{{\mathbb R}}

\newcommand{\OO}{{\mathcal{O}}}

\newcommand{\bsb}{\boldsymbol b}

\newcommand{\bsr}{\boldsymbol r}

\newcommand{\bsw}{\boldsymbol w}
\newcommand{\bsx}{\boldsymbol x}

\newcommand{\bsB}{\boldsymbol B}

\newcommand{\bsD}{\boldsymbol D}
\newcommand{\bsE}{\boldsymbol E}

\newcommand{\bsI}{\boldsymbol I}

\newcommand{\bsR}{\boldsymbol R}

\newcommand{\bsV}{\boldsymbol V}

\newcommand{\bsX}{\boldsymbol X}

\newcommand{\bsZ}{\boldsymbol Z}
\newcommand{\bsone}{\boldsymbol 1}

\newcommand{\bsnull}{\boldsymbol 0}

\newcommand{\bsPP}{\boldsymbol{\mathcal{P}}}

\newcommand{\bsbeta}{\boldsymbol \beta}

\newcommand{\bseta}{\boldsymbol \eta}

\newcommand{\bsvarphi}{\boldsymbol \varphi}
\newcommand{\bspsi}{\boldsymbol \psi}

\DeclareMathOperator*{\argmin}{arg\,min}

\DeclareMathOperator*{\spt}{spt}
\DeclareMathOperator*{\QL}{QL}
\DeclareMathOperator*{\bsQL}{\textbf{QL}}
\DeclareMathOperator*{\CRPS}{CRPS}
\DeclareMathOperator*{\BIC}{BIC}
\DeclareMathOperator*{\lin}{lin}
\DeclareMathOperator*{\conv}{conv}

\DeclareMathOperator*{\sign}{sign}
\DeclareMathOperator*{\softmax}{SoftMax}
\DeclareMathOperator*{\bssoftmax}{\textbf{SoftMax}}

\newcommand{\ov}\overline
\newcommand{\what}{\widehat}
\newcommand{\wtilde}{\widetilde}

\newcommand{\rig}\right
\newcommand{\lef}\left
\newcommand{\nf}\normalfont

\newtheorem{theorem}{Theorem}

\newtheorem{proposition}{Proposition}

\begin{document}

\begin{frontmatter}

    \title{CRPS Learning}

    \author{Jonathan Berrisch\fnref{jmail}}
    \author{Florian Ziel\fnref{fmail}}
    \address{University of Duisburg-Essen, Germany}
    \fntext[jmail]{jonathan.berrisch@uni-due.de}
    \fntext[fmail]{florian.ziel@uni-due.de}
    \begin{abstract}
        \input{00_abstract.tex}
    \end{abstract}
\end{frontmatter}

\section{Introduction}

\input{01_introduction.tex}

\input{03_CRPS_Learning.tex}

\section{Simulation Study}\label{sim}

\input{04_simulation_study.tex}

\section{Application}\label{app}

\input{05_application.tex}\label{appl}

\section{Summary and Conclusion}

\input{06_disc_concl.tex}\label{concl}

\bibliography{library.bib}

\section*{Appendix}

\input{99_appendix.tex}

\end{document}

%% file: 00_abstract.tex
Combination and aggregation techniques can significantly improve forecast accuracy. This also holds for probabilistic forecasting methods where predictive distributions are combined.
There are several time-varying and adaptive weighting schemes
such as Bayesian model averaging (BMA).
However, the quality of different forecasts may vary not only over time but also within the distribution.
For example, some distribution forecasts may be more accurate in the center of the distributions, while others are better at predicting the tails.
Therefore, we introduce a new weighting method that considers the differences in performance over time and within the distribution.
We discuss pointwise combination algorithms based on aggregation across quantiles that optimize with respect to the continuous ranked probability score (CRPS).
After analyzing the theoretical properties of pointwise CRPS learning, we discuss B- and P-Spline-based estimation techniques for batch and online learning, based on quantile regression and prediction with expert advice.
We prove that the proposed fully adaptive Bernstein online aggregation (BOA) method for pointwise CRPS online learning has optimal convergence properties.
They are confirmed in simulations
and a probabilistic forecasting study for European emission allowance (EUA) prices.

%% file: 01_introduction.tex
Combination and aggregation methods are gaining popularity in forecasting, especially in meteorology, energy, economics, finance, and machine learning communities \citep{hsiao2014there,atiya2020does, petropoulos2020forecasting}.
As the predictive accuracy of different forecasting methods may vary over time, several methods learn about this behavior. So there are time-varying and adaptive weighting schemes.
Also, in probabilistic forecasting, where we aim to predict the full predictive distribution or density, many combination methods exist, such as the Bayesian model averaging (BMA) \citep{raftery2005using, fragoso2018bayesian}.
However, individual forecasters may perform differently in different parts of the distribution. For instance, one is more accurate in the center while others are more accurate in the distribution's tails. Thus, intuitively, it makes sense to exploit this potential of varying performance across the distribution, as we will do in this manuscript.

There are various ways to combine distribution function forecasts
to receive a target distribution forecast. The most popular linear approaches are combining across probabilities (vertical aggregation) and combining across quantiles (horizontal aggregation, Vincentization)~\citep{lichtendahl2013better, busetti2017quantile}. Both principles are valid and have their pros and cons, depending on the forecasters' objective. In econometric and statistical literature, forecast combination with respect to the Kullback-Leibler divergence is popular. This approach corresponds to minimizing the log-score of density and usually utilizes the combination across probabilities \citep{gneiting2013combining, aastveit2014nowcasting, kapetanios2015generalised, aastveit2019evolution}.

Other recent publications focus on forecast combination with respect to the continuous ranked probability score (CRPS), a strictly proper scoring rule for distributions, which does not require a density~\citep{gneiting2007strictly}. For instance, \citet{raftery2005using} combine density forecasts using BMA with refinement by minimizing the CRPS. \citet{jore2010combining, opschoor2017combining, thorey2017online, v2020online} and~\citet{zamo2021sequential} use an exponential weighting scheme based on combination across probabilities. A polynomial CRPS learning algorithm is applied in~\citet{thorey2018ensemble}.~\citet{korotin2019integral, korotin2020mixing} analyze properties of exponential weighting learning algorithms for forecasting distributions, including a CRPS based algorithm. Also,~\citet{zhang2020load} show an application for probabilistic electricity demand forecasting where they optimize combination weights with respect to the CRPS. \citet{lin2018multi} and~\citet{van2018probabilistic} apply density-based aggregation methods with weights optimized with respect to the CRPS, the latter for the Gaussian case.

Additionally, plenty of forecast combination applications evaluate the predictive performance of distribution forecasts in terms of CRPS but do not optimize with respect to it. Some authors like~\citet{bai2020does} and~\citet{zamo2021sequential} also compute combination weights based on inverse CRPS-weighting in probabilistic oil price and wind speed forecasting. However, all applications mentioned above consider time-varying weights but ignore the potentially varying performance within the distribution.

It turns out that aggregation across quantiles is a useful tool for considering varying performance within our forecasters' distribution when aiming for optimal CRPS. Therefore, instead of searching for optimal weights, we are looking for optimal weight functions. Furthermore, the weight functions shall be chosen pointwise so that the resulting distribution forecast has minimal CRPS. Thus combination weights assigned to a forecast can be larger in areas of the distribution where this forecast performs better.

This manuscript contributes to the probabilistic forecasting literature, as we
\begin{itemize}
    \item[i)] Introduce a probabilistic forecast combination framework for pointwise CRPS learning based on aggregation across quantiles.
    \item[ii)] Prove that the proposed pointwise CRPS learning outperforms standard CRPS learning methods.
    \item[iii)] Study combination techniques based on B-Splines and P-Splines to model the combination weight functions.
    \item[iv)] Analyze a fully adaptive CRPS based Bernstein online aggregation (BOA) algorithm that satisfies theoretical guarantees on the convergence rates.
    \item[v)] Conduct simulation studies and show an application for probabilistic forecasting of European emission allowance (EUA) prices.
\end{itemize}

In the proceeding Section~\ref{exp_adv}, we introduce the relevant basics on forecast combination, batch and online learning, and risk. We continue in Section~\ref{sec_crps_learning} by studying the basics of the CRPS, introducing CRPS learning and its optimality properties, and discussing batch estimation techniques that utilize quantile regression. In Section~\ref{theor}, we briefly review the online learning concept of prediction with expert advice. Then we introduce the CRPS based BOA algorithm and provide a theorem for the optimality of convergence rates. Section~\ref{sec_extentions} discusses extensions to the CRPS learning algorithm, e.g., by considering P-Splines, and includes implementation remarks.
In Section~\ref{sim}, we present simulation studies that evaluate the properties of the proposed algorithm and compare the performance with other existing algorithms. Finally, we illustrate the procedure by applying it to probabilistic forecasting of European emission allowances prices in Section~\ref{appl}. Section~\ref{concl} summarizes and concludes.

%% file: 03_CRPS_Learning.tex
\section{Relevant Aspects on Forecast Combination}\label{exp_adv}

\subsection{Notation and Basics}

We first require some notation. Let $K$ be the number of forecasters (often referred to as experts) that we want to aggregate. Denote $\what{X}_{t,k}$ the forecast of expert $k$ for the prediction target $Y_t$ at time $t$. This forecast $\what{X}_{t,k}$ could be a forecast for the mean or median of $Y_t$. $\what{X}_{t,k}$ could also be the distribution function of $Y_t$ or its inverse. We assume that $\what{\bsX}_{t} = (\what{X}_{t,1},\ldots,\what{X}_{t,K})'$ is predicted sequentially given all information $\mathcal{F}_{t-1}$ available until $t-1$.

We want to use weights $w_{t,k}$ at time $t$ for expert $k$ to define the linear forecast combination $\wtilde{X}_t$ by
\begin{equation}
    \wtilde{X}_{t} =  \sum_{k=1}^K w_{t,k} \what{X}_{t,k}
    = \bsw_t' \what{\bsX}_{t}
    \label{eq_forecast_def}
\end{equation}
Many applications restrict $w_t$ to convex combinations, i.e. $\sum_{k=1}^K w_{t,k} = 1$ and $w_{t,k}\geq 0$. A common aggregation procedure is the naive combination method, also known as uniform aggregation or simple average where $w_{t,k}^{\text{naive}}= 1/K$. More sophisticated aggregation or combination rules assign the weights based on the prediction target, but also the past performance or the model fit.

In general, all combination rules for $\bsw_{t}$ can be distinguished in batch learners and online learners. Batch learning algorithms potentially evaluate all past information (e.g.  $\what{\bsX}_{t}, \what{\bsX}_{t-1},\ldots$, $Y_{t}, Y_{t-1},\ldots$ and $\bsw_{t-1}, \bsw_{t-2}, \ldots$) for updating $\bsw_{t}$. However, online learners, only evaluate the most recent information (e.g. $\what{\bsX}_{t}$, $ Y_{t-1}$, $\bsw_{t-1}$) to compute $\bsw_{t}$. Obviously, the latter algorithms tend to be substantially faster in practice as less information has to be evaluated. Sometimes, both approaches can be converted into each other~\citep{hoi2018online}. A standard batch learning example is ordinary least squares (OLS). In a rolling or expanding window context the corresponding online learning counterpart is recursive least squares (RLS)~\citep{lee1981recursive}. Another popular online learning procedure is the Kalman filter, which corresponds to a ridge regression problem in batch learning \citep{diderrich1985kalman}.

For the prediction target, there might be proper scoring rules available~\citep{gneiting2007strictly}. A scoring or loss function $\ell$ takes as first input a forecast $\what{X}_{t,k}$ and as a second input the prediction target $Y_t$.  The loss function is strictly proper, if it can identify the true prediction target. Popular loss functions used for point forecasting include the $\ell_2$-loss $\ell_2(x, y) = | x -y|^2$ which is strictly proper for mean predictions  and the $\ell_1$-loss $\ell_1(x, y) = | x -y|$ which is strictly proper for median predictions~\citep{gneiting2011making}. For distribution functions the CRPS is a strictly proper scoring rule\footnote{The underlying distribution has to be in $L_1$, so it requires a finite first moment.} which we consider in Section~\ref{sec_crps_learning}.

\subsection{Risk and its Optimality}

According to~\citet[p. 120]{wintenberger2017optimal}, the risk is a suitable measure for predictive accuracy. The risk $\wtilde{\mathcal{R}}_t$ of our forecast combination $\wtilde{X}_{t}$ is given by
\begin{align}
    \wtilde{\mathcal{R}}_t = \sum_{i=1}^t \mathbb{E}[\ell(\wtilde{X}_{i},Y_i)|\mathcal{F}_{i-1}]
    \label{eq_risk}
\end{align}
where $\mathcal{F}_{t}$ contains all information available at time $t$. For simplicity we will denote the risk of expert $\what{X}_{t,k}$ by $\what{\mathcal{R}}_{t,k}$  which can be computed with~\eqref{eq_risk} by choosing the $k$-th unit vector for $w_{i,k}$ in the definition~\eqref{eq_forecast_def}.

A suitable combination rule may satisfy two properties in iid settings.\footnote{In the iid setting we assume that $(Y_t, \bsZ_t)$ is iid where $\bsZ_t$ is some external input, such that $\what{X}_{t,k}$ depends on $(Y_{t-1},\bsZ_{t-1}), (Y_{t-2},\bsZ_{t-2}),\ldots$ with corresponding filtration $\mathcal{F}_{t-1}$} The first is the so called model selection property. This means that the risk of our forecast combination $\wtilde{X}_i$ is asymptotically not worse than risk of the best individual expert $\what{\mathcal{R}}_{t,\min} = \min_{k=1,\ldots,K} \what{\mathcal{R}}_{t,k}$.

Even stricter properties are called linear and convex aggregation properties. In the former, the performance of our forecast combination $\wtilde{X}_t$ is compared to the best linear combination
$\what{X}_{t,\lin} = \bsw_{t,\lin}'\what{\bsX}_t $
with aggregation weights $\bsw_{t,\lin} = \argmin_{\bsw \in {\R}^K} \sum_{i=1}^t \mathbb{E}[\ell( \bsw' \what{\bsX}_{i}, Y_i)|\mathcal{F}_{i}]$ and risk
$\what{\mathcal{R}}_{t,\lin} = \sum_{i=1}^t \mathbb{E}[\ell(\what{X}_{i,\lin}, Y_i)|\mathcal{F}_{i}]$. In the convex aggregation property
it is compared to the best convex combination
$\what{X}_{t,\conv} = \bsw_{t,\conv}'\what{\bsX}_t $
with $w_{t,\conv} = \argmin_{\bsw \in {\R}^K, \bsw\geq \bsnull, \bsw'\bsone=1} \sum_{i=1}^t \mathbb{E}[\ell( \bsw' \what{\bsX}_{i}, Y_i)|\mathcal{F}_{i}]$ and risk
$\what{\mathcal{R}}_{t,\conv} = \sum_{i=1}^t \mathbb{E}[\ell(\what{X}_{i,\conv}, Y_i)|\mathcal{F}_{i}]$.

The convex combination can never be worse than the best individual expert. Similarly, the linear combination will never be worse than the best convex one. It turns out that satisfying the convex and linear property comes at the cost of slower possible convergence. The optimal convergence rate in $t$ for the linear and convex properties is $\mathcal{O}(\sqrt{t})$ which is standard in asymptotic statistics while the selection property can yield the fast rate $\mathcal{O}(t)$.

Based on~\citet[p.121]{wintenberger2017optimal}, an algorithm has optimal rates if
\begin{align}
    \frac{1}{t}\left(\wtilde{\mathcal{R}}_t - \what{\mathcal{R}}_{t,\min} \right)  & =
    \mathcal{O}\left(\frac{\log(K)}{t}\right)\label{eq_optp_select}                    \\
    \frac{1}{t}\left(\wtilde{\mathcal{R}}_t - \what{\mathcal{R}}_{t,\conv} \right) & =
    \mathcal{O}\left(\sqrt{\frac{\log(K)}{t}}\right)
    \label{eq_optp_conv}                                                               \\
    \frac{1}{t}\left(\wtilde{\mathcal{R}}_t - \what{\mathcal{R}}_{t,\lin} \right)  & =
    \mathcal{O}\left(\sqrt{\frac{\log(K)}{t}}\right)
    \label{eq_optp_lin}
\end{align}
which characterises optimal dependence on time $t$ and the number of experts $K$. Learning algorithms may satisfy~\eqref{eq_optp_select},~\eqref{eq_optp_conv} and even~\eqref{eq_optp_lin} depending on the loss function $\ell$ and regularity conditions on $Y_t$ and $\wtilde{X}_{t,k}$. A key contribution of this manuscript is a fully adaptive online learning algorithm that almost satisfies~\eqref{eq_optp_select} and~\eqref{eq_optp_conv} in a CRPS learning setting.

\section{Fundamentals on CRPS Learning}\label{sec_crps_learning}

\subsection{Basics on the CRPS}

The CRPS is a scoring rule for evaluating probabilistic forecasts with respect to corresponding observations. The CRPS is defined as follows:
\begin{align}
    \CRPS(F, y) & = \int_{{\R}} {(F(x) - \mathbb{1}\{ x > y \})}^2 dx
    \label{eq_crps}
\end{align}
where $F$ is the distribution function and $y$ the observation. The CRPS is strictly proper and widely used in forecasting~\citep{gneiting2014probabilistic, jordan2019evaluating}. Now, we want to use an important characterization of the CRPS. To utilize it, we require $F^{-1}$ to be uniquely defined (e.g. $F$ has a pdf). It is based on the quantile loss (QL, also known as quantile score, pinball score, and asymmetric piecewise linear score), which is a strictly proper scoring rule for quantile forecasts~\citep{gneiting2011making}. The quantile loss is defined as
\begin{align}
    {\QL}_p(q, y) & = (\mathbb{1}\{y < q\} -p)(q - y)
\end{align}
for a probability $p\in (0,1)$, quantile forecast $q$ at probability $p$ for the target value $y$.
The CRPS can be written as scaled integral on ${\QL}_p$ over $p$~\citep[p. 414]{gneiting2011comparing}:
\begin{align}
    \CRPS(F, y) = 2 \int_0^{1}  {\QL}_p(F^{-1}(p), y) \, d p.
    \label{eq_crps_qs}
\end{align}
Representation~\eqref{eq_crps_qs} has a clear advantage. It allows us to move from distribution-based evaluation procedures to an integral of quantile forecasts, which is easier to handle~\citep{gneiting2011making, gneiting2011quantiles}.

We can approximate the CRPS by
\begin{align}
    \CRPS(F, y) \approx \frac{2}{M} \sum_{m=1}^M  {\QL}_{p_i}(F^{-1}(p_i), y)
    \label{eq_crps_approx}
\end{align}
for an equidistant dense grid $\bsPP = ( p_1,\ldots, p_M )$ with $p_m<p_{m+1}$ and $p_{m+1}-p_m = h_M$ for all $m$. Clearly, $M\to \infty$ induces $h_M \to0$, $p_1\to0$, $p_M\to1$ and the approximation converges to the CRPS.

One problem in practice is that explicit solutions for the CRPS are only available for limited parametric cases \citep{jordan2019evaluating}. Additionally, except for mentioned limited parametric cases, reporting the exact full predictive distributions is impossible in many applications. It always requires some degree of approximation. A typical option
for reporting $\what{F}$ is to report many quantiles $\what{F}^{-1}(p)$ on a dense grid of probabilities $p$ as in the CRPS approximation~\eqref{eq_crps_approx}.
A popular choice for such a grid $\bsPP$ contains all percentiles, i.e. $\bsPP=(0.01,0.02,\ldots, 0.99 )$, see e.g. \citet{HONG2016896}.
The usage of this reporting method has the additional advantage, that next to methods that forecast $\what{F}$
other forecast methods that only predict the quantiles on a dense grid $\bsPP$ are applicable. The latter is particularly relevant for quantile regression methods.

\subsection{CRPS learning and its optimality}

Now, we discuss the learning theory with respect to the CRPS~\eqref{eq_crps_qs}. As pointed out in the introduction, \citet{thorey2018ensemble} and \citet{v2020online} consider CRPS learning by aggregating across probabilities
\begin{equation}
    \wtilde{F}_{t} = \sum_{k=1}^K w_{t,k} \what{F}_{t,k}
    \label{eq_comb_crps_cw_prop}
\end{equation}
for some combination weights $w_{t,k}\in[0,1]$ and expert forecasts $\what{F}_{t,k}$. As far as we know, there is no consideration of methods that aggregate across quantiles
\begin{equation}
    \wtilde{F}_{t}^{-1} = \sum_{k=1}^K w_{t,k} \what{F}^{-1}_{t,k}
    \label{eq_comb_crps_cw_quant}
\end{equation}
for CRPS learning. Both aggregation rules~\eqref{eq_comb_crps_cw_prop} and~\eqref{eq_comb_crps_cw_quant} can be used to evaluate combination methods with respect to the optimality properties~\eqref{eq_optp_select},~\eqref{eq_optp_conv} and~\eqref{eq_optp_lin}.

As we evaluate our predictions by the CRPS a consistent choice for the loss $\ell$ is the CRPS itself. For the corresponding risk $\mathcal{R}_{t,k}$ in~\eqref{eq_risk}, we receive with~\eqref{eq_crps_qs} and Fubini
\begin{align}
    \mathcal{R}_{t,k}^{\CRPS} & =  \sum_{i = 1}^t \mathbb{E}[ \CRPS(\wtilde{F}_{i},Y_i) | \mathcal{F}_{i-1} ]                                                                                                             \\
                              & =  \sum_{i = 1}^t \mathbb{E}\left[ \left. 2 \int_{0}^1 {\QL}_p(\wtilde{F}_{i}^{-1}(p),Y_i) \right| \mathcal{F}_{i-1} \right ]     \, dp                                                   \\
                              & =    \int_{0}^1  \underbrace{2 \sum_{i = 1}^t  \mathbb{E}\left[ \left. {\QL}_p(\wtilde{F}_{i}^{-1}(p),Y_t) \right| \mathcal{F}_{i-1} \right ]  }_{=2\mathcal{R}_{t,k}^{{\QL}}(p)} \, dp .
    \label{eq_risk_crps}
\end{align}
This shows that the CRPS based risk $\mathcal{R}^{\CRPS}_{t,k}$ can be represented as integral over the pointwise risk $\mathcal{R}_{t,k}^{\QL}(p)$ scaled by 2. Using~\eqref{eq_risk} the latter can be computed as well by
\begin{align}
    \mathcal{R}_{t,k}^{\QL}(p)
     & =   \sum_{i = 1}^t  \mathbb{E}\left[ \left. {\QL}_p(\wtilde{F}_{i}^{-1}(p),Y_i)  \right| \mathcal{F}_{i-1} \right ] 
\end{align}
where $\ell(F,y;p) = {\QL}_p(F^{-1}(p),y)$.
The calculations show, that~\eqref{eq_comb_crps_cw_quant} is a useful baseline for pointwise CRPS learning.
Hence, we introduce weight functions $w_{t,k}$ on $(0,1)$ at time $t$ for expert $k$ to define the forecast $\wtilde{F}$ pointwise by
\begin{equation}
    \wtilde{F}_{t}^{-1}(p) = \sum_{k=1}^K w_{t,k}(p) \what{F}^{-1}_{t,k}(p)
    \label{eq_forecast_F_def}
\end{equation}
for all $p\in (0,1)$. Now, we are interested in finding $w_{t,k}$ such that $\wtilde{F}_t$ has minimal CRPS. Clearly, for constant $w_{t,k}$ in~\eqref{eq_forecast_F_def} we are in the standard setting~\eqref{eq_comb_crps_cw_quant}.

When analysing the optimality of CRPS learning algorithms, we can consider~\eqref{eq_risk}
on a pointwise basis for the quantile loss. Thus,
we denote the risks by
$\wtilde{\mathcal{R}}^{\QL}_{t}(p)$,
$\what{\mathcal{R}}^{\QL}_{t,\min}(p)$,
$\what{\mathcal{R}}^{\QL}_{t,\conv}(p)$ and
$\what{\mathcal{R}}^{\QL}_{t,\lin}(p)$
for the forecast combination, the best expert and the best convex and linear combination based on the quantile loss and quantile forecasts $\what{F}^{-1}_{t,k}(p)$ for probability $p\in(0,1)$.
For the evaluation it is useful to consider the corresponding scaled integrals as joint measures
\begin{align}
    \overline{\wtilde{\mathcal{R}}}^{\QL}_{t} = 2 \int_0^1
    \wtilde{\mathcal{R}}^{\QL}_{t}(p) \, dp
    \text{ , }
    \overline{\what{\mathcal{R}}}^{\QL}_{t,\min} = 2 \int_0^1
    \what{\mathcal{R}}^{\QL}_{t,\min}(p) \, dp , \\
    \overline{\what{\mathcal{R}}}^{\QL}_{t,\conv} = 2 \int_0^1
    \what{\mathcal{R}}^{\QL}_{t,\conv}(p) \, dp
    \text{ and }
    \overline{\what{\mathcal{R}}}^{\QL}_{t,\lin} = 2 \int_0^1
    \what{\mathcal{R}}^{\QL}_{t,\lin}(p) \, dp
\end{align}
for each $p\in (0,1)$ which we scale by 2 due to~\eqref{eq_crps_qs}.
Further, we denote the cumulative risks of the forecast combination, the best expert and the best convex and linear combination based on the CRPS and distribution forecasts $\what{F}_{t,k}$ by
$\wtilde{\mathcal{R}}^{\CRPS}_{t}$,
$\what{\mathcal{R}}^{\CRPS}_{t,\min}$,
$\what{\mathcal{R}}^{\CRPS}_{t,\conv}$ and
$\what{\mathcal{R}}^{\CRPS}_{t,\lin}$ respectively. As pointwise optimization looks for optimalities on larger spaces we have the following Proposition:
\begin{proposition}\label{proposition_optimality}
    It holds
    \begin{align}
        \overline{\what{\mathcal{R}}}^{\QL}_{t,\min}
         & \leq \what{\mathcal{R}}^{\CRPS}_{t,\min},
        \label{eq_risk_ql_crps_expert}                 \\
        \overline{\what{\mathcal{R}}}^{\QL}_{t,\conv}
         & \leq \what{\mathcal{R}}^{\CRPS}_{t,\conv} ,
        \label{eq_risk_ql_crps_convex}                 \\
        \overline{\what{\mathcal{R}}}^{\QL}_{t,\lin}
         & \leq \what{\mathcal{R}}^{\CRPS}_{t,\lin}
        \label{eq_risk_ql_crps_lin}.
    \end{align}
    The inequalities are strict if and only if the weights
    $w^*_{t,k}$ associated to
    $\what{\mathcal{R}}^{\QL}_{t,*}$ are not constant for any $k$, providing that the integrals exist.
\end{proposition}
Thus, pointwise procedures which have the risk $\overline{\what{\mathcal{R}}}^{\QL}_{t,*}$ always have the potential to outperform standard procedures with constant weights functions and risk $\what{\mathcal{R}}^{\CRPS}_{t,*}$. But, of course, no improvement can be expected if the optimal weight functions are constant. Otherwise, the inequalities of Proposition~\ref{proposition_optimality} will be strict.

If we would like to know in advance if it is worth considering the pointwise combination, then we require a pointwise evaluation with respect to the quantile loss on selected probabilities resp. parts of the distributions.
For instance, suppose we want to combine two experts A and B. Then, suppose expert A performs significantly better in one part of the distribution and expert B in another part. In that case, it is a clear sign that a pointwise combination should be considered.
Sometimes, no explicit formal evaluation of the experts is required because the model characteristics and assumptions imply different predictive accuracy in some distribution parts. For instance, this is the case if two experts A and B have different supports. If, for example, A is always positive and B might take negative values, then the predictive performance for lower quantiles must differ, and the pointwise combination is advantageous.

\subsection{CRPS learning using quantile regression}\label{subsec_est_weight_functions}

We have seen that pointwise CRPS learning has the potential to outperform standard CRPS learning methods. Still, the optimal pointwise weights in~\eqref{eq_forecast_F_def} have to be estimated. Theoretically, a pointwise approach has to be applied to all $p\in(0,1)$ such that the weight function $w_{t,k}$ can be specified. As we can never evaluate infinitely many values for $p$ we have to consider some finite-dimensional representation for the weight functions $w_{t,k}$ that are bounded on the unit interval $(0,1)$. A suitable option is to represent the weight functions $w_{t,k}$ by a finite-dimensional representation using basis functions.

So let $\bsvarphi=(\varphi_1,\ldots, \varphi_L)$ be bounded basis functions on $(0,1)$ with $\sum_{i=1}^L \varphi_i \equiv 1$. A popular choice are B-Splines as local basis functions as they allow fast computation due to sparsity. The weight functions $w_{t,k}$ are then represented by
\begin{equation}
    w_{t,k} = \sum_{l=1}^L \beta_{t,k,l} \varphi_l = \bsbeta_{t,k}'\bsvarphi
\end{equation}
with parameter vector $\beta_{t,k}$. Given $T$ historic forecasts $\what{F}^{-1}_{i,k}(p)$ we can estimate the $K\times L$-dimensional parameter matrix $\bsbeta_{t}$ by minimizing the corresponding CRPS using~\eqref{eq_crps_qs}
\begin{align}
    \bsbeta_{t}^{\bsvarphi\text{-CRPS}}
     & = \argmin_{ \bsbeta \in {\R}^{K\times L}} \sum_{i=t-T+1}^t
    \int_0^{1}  {\QL}_p\left( \sum_{k=1}^K \sum_{l=1}^L  \beta_{k,l}  \varphi_l(p) \what{F}^{-1}_{i,k}(p) , Y_i \right) \, d p. \\
     & = \argmin_{ \bsbeta \in {\R}^{K\times L}} \sum_{i=t-T+1}^t
    \int_0^{1}  {\rho}_p\left( Y_i - \sum_{k=1}^K \sum_{l=1}^L  \beta_{k,l}  \varphi_l(p)  \what{F}^{-1}_{i,k}(p)  \right) \, d p.
    \label{eq_qr_basis}
\end{align}
where the second line uses the shift invariance of the quantile loss and quantile regression notation $\rho_p(z) = \QL_p(0,z)  = z(p-\mathbb{1}\{z< 0\})$, \citet{koenker2017handbook}.

Still, computing~\eqref{eq_qr_basis} requires the evaluation of all distribution forecasts. As discussed this is often not possible in practice. If we restrict the evaluation to the grid of probabilities $\bsPP$ problem~\eqref{eq_qr_basis} simplifies with~\eqref{eq_crps_approx} to
\begin{align}
    \bsbeta_{t}^{\bsvarphi\text{-QR}}
    = \argmin_{ \bsbeta \in {\R}^{K\times L}} \sum_{i=t-T+1}^t  \sum_{p\in\bsPP} \rho_p \left( Y_i -  \sum_{k=1}^K \sum_{l=1}^L  \beta_{k,l} \varphi_l(p) \what{F}^{-1}_{i,k}(p) \right)
    \label{eq_qr_basis_p}.
\end{align}

In general, quantile regression problems can be solved efficiently using linear programming solvers \citep{koenker2017handbook}. However,~\eqref{eq_qr_basis_p} is not a simple quantile regression problem, but a joint quantile regression~\citep{sangnier2016joint, chun2016graphical}. The parameters $\beta_{k,l}$ are active for multiple quantiles. Thus, adequate estimation of $\beta_{k,l}$ requires solving the optimization problem with respect to $K\times L$ parameters which can be computationally costly if $K$ and $L$ are large.

However, if we choose the basis $\bsvarphi$ so that $\varphi_i=\mathbb{1}{\{p_i\}}$ on $\bsPP$ for $p_i\in \bsPP=(p_1,\ldots, p_M)$ then~\eqref{eq_qr_basis_p} can be disentangled into $L$ separate quantile regression problems.
This is
\begin{align}
    \bsw^{\text{QR}}_{t}(p)
    = \argmin_{ \bsw \in {\R}^K} \sum_{i=t-T+1}^t \rho_p \left( Y_i -  \sum_{k=1}^K w_{k} \what{F}^{-1}_{i,k}(p) \right)
    \label{eq_qr}
\end{align}
for $p\in (0,1)$
where $\what{F}^{-1}_{t,k}(p)$ are the experts for the $p$-quantile.

Quantile regression~\eqref{eq_qr} will lead to linear optimality~\eqref{eq_risk_ql_crps_lin} on $\bsPP$, as long as standard regularity conditions required for the quantile regression are satisfied \citep{koenker2001quantile}. However, we might assume further restrictions to reduce the estimation risk, e.g., the solution is a convex combination. \citet{taylor1998combining} discussed many related plausible restrictions for quantile combination concerning bias correction, positivity, affinity, among others.

\subsection{Quantile crossing}\label{subsec_quantile_crossing}

All pointwise combination methods for the CRPS (as e.g.~\eqref{eq_qr}) suffer from the potential problem of \textit{quantile crossing}~\citep{chernozhukov2010quantile}. This problem occurs if we have $\wtilde{F}_{t}^{-1}(p_i) > \wtilde{F}_{t}^{-1}(p_j)$ for some $p_i,p_j\in(0,1)$ with $p_i<p_j$.

One way to solve this problem is to impose the monotonicity constraint on the optimization objective~\citep{bondell2010noncrossing}.
Another suitable and universal approach is to rearrange (i.e., sort) the predictions so that no quantile crossing occurs.
As shown in~\citet{chernozhukov2010quantile}, this procedure will never reduce the forecasting performance in terms of the quantile loss. Thus, we recommend utilizing it always in practice.

However, we remark that for the special case of constant weight functions $w_{t,k}$, quantile crossing is no issue if no expert suffers from quantile crossing. So they satisfy
$\what{F}_{t,k}^{-1}(p_i) \leq \what{F}_{t,k}^{-1}(p_j)$ for $p_i,p_j\in(0,1)$ with $p_i<p_j$, and thus it holds
$    \wtilde{F}_{t}^{-1}(p_i) = \sum_{k=1}^K w_{t,k} \what{F}^{-1}_{t,k}(p_i)
    \leq \sum_{k=1}^K w_{t,k} \what{F}^{-1}_{t,k}(p_j)
    = \wtilde{F}_{t}^{-1}(p_j)
$.

\section{Online CRPS Learning}\label{theor}

We have seen that quantile regression is a suitable approach for CRPS learning. Still, it is a batch learning approach that might be computationally costly. Therefore, we study online learning methods for CRPS learning. First, we cover relevant aspects of online learning methods usually referred to as prediction with expert advice, which is discussed in detail in~\citet{cesa2006prediction}. This literature is driven a lot by game theory and optimization~\citep{hoi2018online}. Most of the results hold for deterministic time series but hold similarly in a stochastic setting that we consider~\citep{wintenberger2017optimal}. However, there are also connections to econometric and statistical literature that we will briefly mention as well.

\subsection{Basics on prediction with expert advice}

A key element in~\citet{cesa2006prediction} is the so called regret $R_{t,k}$ that measures the predictive accuracy of expert $k$ until time $t$. It evaluates the past performance of the experts $\what{X}_{t,k}$ compared to the performance of the forecaster $\wtilde{X}_{t}$ with respect to a target loss $\ell$
\begin{equation}
    R_{t,k} = \sum_{i = 1}^t r_{t,k} =
    \sum_{i = 1}^t \ell(\wtilde{X}_{i},Y_i) - \ell(\what{X}_{i,k},Y_i)
    \label{eq_regret}
\end{equation}
with the instantaneous regret $r_{t,k} = \ell(\wtilde{X}_{t},Y_t) - \ell(\what{X}_{t,k},Y_t)$.
In non-stochastic settings the regret satisfies $R_{t,k} = \wtilde{\mathcal{R}}_t  - \what{\mathcal{R}}_{t,k}$, see~\eqref{eq_risk}.

Given the regret vector $\bsR_{t}= (R_{t,1},\ldots,R_{t,K})'$ various aggregation rules can be defined. Popular aggregation rules include the polynomial weighted aggregation (PWA)~\citet[p. 12]{cesa2006prediction} and the exponentially weighted aggregation (EWA, also called exponentially weighted average forecaster or Hedge algorithm)~\citet[p. 14]{cesa2006prediction}. The EWA utilizes the $\softmax$ operator on the regret vector, scaled by a learning rate $\eta>0$ which may be regarded as a tuning parameter. Thus, the EWA is a combination rule defined for prior weight vector $\bsw_0$ by
\begin{align}
    w_{t,k}^{\text{EWA}}  & = Kw_{0,k} \frac{e^{\eta R_{t,k}} }{\sum_{j = 1}^K e^{\eta R_{t,j}}}
    =
    \frac{e^{-\eta \ell(\what{X}_{t,k},Y_t)} w^{\text{EWA}}_{t-1,k} }{\sum_{j = 1}^K e^{-\eta \ell(\what{X}_{t,j},Y_t)} w^{\text{EWA}}_{t-1,j} } ,
    \\
    \bsw_{t}^{\text{EWA}} & =
    K\bsw_0 \odot \softmax(\eta \bsR_{t})
    = \softmax( -\eta \ell(\what{\bsX}_{t},Y_t) + \log(\bsw^{\text{EWA}}_{t-1})  )
    \label{eq_ewa_general}
\end{align}
where the second line corresponds to the vector notation of the line above with coordinatewise evaluation of involved multiplications and functions, and $\odot$ as Hadamard product. The EWA is also analyzed and applied in econometrics and statistics literature~\citep{jore2010combining, dalalyan2012sharp, opschoor2017combining}. Due to the positivity of the exponential function $\bsw_{t}^{\text{EWA}}$ will always be a convex combination.

As shown in~\citet[p. 17]{cesa2006prediction}, \citet{kakade2008generalization} and \citet{gaillard2014second}, EWA~\eqref{eq_ewa_general} satisfies selection optimality~\eqref{eq_optp_select} if the loss $\ell$ is exp-concave and $\eta$ is chosen correctly. This holds for example for the $\ell_2$-loss and $Y_t$ bounded on $[0,1]$ with $\eta_t = \sqrt{8 \log(K)/t}$.

To achieve convex aggregation optimality~\eqref{eq_optp_conv} for a learning algorithm, usually the gradient trick is applied~\citet[p. 22]{cesa2006prediction}. For the gradient trick, we define the linearized version of the loss as $\ell^{\nabla}(x,y) = \ell'(\wtilde{X},y) x$ where $\ell'$ is the subgradient of $\ell$ in its first coordinate evaluated at forecast combination $\wtilde{X}$\footnote{This linearized version is sometimes also called pseudo-loss function~\citet{devaine2013forecasting}.}. Replacing $\ell(x,y)$ in~\eqref{eq_ewa_general} by $\ell^{\nabla}(x,y)$ yields the gradient version of the EWA. This gradient based EWA (EWAG) algorithm (often referred as exponentiated gradient aggregation, EGA) with exp-concave losses and bounded $Y_t$ satisfies both~\eqref{eq_optp_select} and~\eqref{eq_optp_conv}. Unfortunately, the quantile loss that is a key element in CRPS learning is not exp-concave. Hence, we study a more advanced algorithm.

\subsection{Bernstein online aggregation (BOA)}

Some more advanced algorithms are based on the EWA~\eqref{eq_ewa_general} and use a second-order refinement based on the variance estimate
\begin{equation}
    V_{t,k} =  \sum_{i = 1}^t r_{i,k}^2 =
    \sum_{i = 1}^t \left( \ell(\wtilde{X}_{i},Y_i) - \ell(\what{X}_{i,k},Y_i) \right)^2
    \label{eq_variance}
\end{equation}
of the instantaneous regret to receive better convergence and stability properties. For instance,~\citet{koolen2015second} and~\citet{wintenberger2017optimal} introduce such algorithms. In the former, this is called Squint; in the latter, Bernstein online aggregation (BOA). Both algorithms share very similar properties and coincide in specific situations~\citep{gaillard2018efficient, mhammedi2019lipschitz}. We focus on the BOA, as it was designed for stochastic settings that we consider. The BOA algorithm~\citet[p. 121]{wintenberger2017optimal} is defined by
\begin{align}
    \bsw_{t}^{\text{BOA}} & =  K \bsw_{0} \odot
    \softmax \left( \eta \bsR_t (1-\eta \bsV_t)  \right) \\
                          & =
    \softmax \left( -\eta \ell(\what{\bsX}_{t},Y_t)(1-\eta \ell(\what{\bsX}_{t},Y_t)) + \log( \bsw^{\text{BOA}}_{t-1} ) \right)
    \label{eq_boa}
\end{align}
with $\bsV_{t}=(V_{t,1},\ldots,V_{t,K})'$. As discussed above, the gradient trick can be applied by replacing $\ell$ with $\ell^{\nabla}$ leading to the gradient-based BOA algorithm (BOAG).

The second-order refinement in the BOA(G) compared to the EWA(G) along with adequate tuning of $\eta$ allows weakening the exp-concavity assumption on the loss while retaining optimality with respect to~\eqref{eq_optp_select} and~\eqref{eq_optp_conv}.
For example, \citet{wintenberger2017optimal}
proves an estimate for BOAG under convex losses that depends on $\eta$. Adequately setting $\eta$ requires two ingredients.

First, it requires multiple learning rates $\eta_k$ such that
they can be optimized individually for each expert $k$.
The second requirement is time-adaptive learning rates based on the so-called doubling trick~\citep[p. 17]{cesa2006prediction}.
The time adaptivity of the learning rates $\eta_{t,k}$ depends on a range estimate $\bsE_t$ of the regret and the variance estimate $\bsV_t$.
The fully adaptive BOAG algorithm with multiple learning rates from~\citet[p. 124]{wintenberger2017optimal} contains both ingredients:\footnote{Note that we consider a minor adjustment concerning some rounding issues in comparison to~\citet{wintenberger2017optimal} which does not change the algorithm's convergence properties.}
\begin{subequations}
    \begin{align}
        \bsr_{t}   & = \ell^{\nabla}(\wtilde{X}_{t},Y_t) - \ell^{\nabla}(\what{\bsX}_{t},Y_t)                                                                         \\
        \bsE_{t}   & = \max(\bsE_{t-1}, \bsr_{t}^+ + \bsr_{t}^-)                                                                                                      \\
        \bsV_{t}   & = \bsV_{t-1} + \bsr_t^{ \odot 2}                                                                                                                 \\
        \bseta_{t} & =\min\left( \left(-\log(\bsw_0) \odot \bsV_t^{\odot -1} \right)^{\odot\frac{1}{2}} ,  \frac{1}{2}\bsE_{t}^{\odot-1}\right)                       \\
        \bsR_{t}   & = \bsR_{t-1}+  \bsr_{t} \odot \left(  \bsone - \bseta_{t} \odot \bsr_{t} \right)/2 + \bsE_{t} \odot \mathbb{1}\{-2\bseta_{t}\odot \bsr_{t} > 1\} \\
        \bsw_{t}   & = K \bsw_{0} \odot \softmax\left( -  \bseta_{t} \odot \bsR_{t} + \log( \bseta_{t}) \right)\end{align}
    \label{algo:boa_general_step}
\end{subequations}
using $\odot$ also in Hadamard power notation, $\bsx^+$ and $\bsx^-$ as elementwise positive and negative part of $\bsx$, and elementwise evaluation of $\max$, $\min$ and $\log$.

We remark that~\citet[Theorem 2.1]{gaillard2018efficient} prove that the fully adaptive BOAG satisfies selection optimality under mild regularity conditions: There exist a $C>0$ such that for $x>0$ it holds that
\begin{equation}
    P\left( \frac{1}{t}\left(\wtilde{\mathcal{R}}_t - \what{\mathcal{R}}_{t,\min} \right) \leq C\left(\frac{\log(K)+\log(\log(Gt))+ x}{\alpha t}\right)^{\frac{1}{2-\beta}} \right) \geq
    1-2e^{-x}
    \label{eq_boa_opt_select}
\end{equation}
if the considered loss $\ell(\cdot, Y_t)$ is convex $G$-Lipschitz and weak exp-concave in its first coordinate (see  Appendix, Proof of Theorem~\ref{proof_theorem} for a formal definition, including the role of $\alpha$). It is clear that for $\beta=1$ the algorithm satisfies~\eqref{eq_optp_select}, despite for the $\log(\log(t))$ term. This $\log(\log(t))$ term is the standard price we have to pay for online calibration of the learning rates $\bseta_t$~\citep[p. 4]{gaillard2014second}. Furthermore, $\beta=0$ is satisfied by every Lipschitz convex loss, and any $\beta>0$ leads to fast convergence rates.

Furthermore,~\citet{wintenberger2017optimal} proves in Theorem 4.3 that the same BOA algorithm satisfies that there exist a $C>0$ such that for $x>0$ it holds that
\begin{equation}
    P\left( \frac{1}{t}\left(\wtilde{\mathcal{R}}_t - \what{\mathcal{R}}_{t,\conv} \right)  \leq C \log(\log(t)) \left(\sqrt{\frac{\log(K)}{t}} + \frac{\log(K)+x}{t}\right)  \right) \geq
    1-e^{-x}.
    \label{eq_boa_opt_conv}
\end{equation}
if the loss $\ell$ is convex and the variance satisfies $\frac{1}{t}V_{t,k} \to V_k < V  < \infty$. Thus, the algorithm is almost optimal with respect to~\eqref{eq_optp_conv}, only the $\log(\log(t))$ term leads to a small lack of optimality.

\subsection{CRPS online learning and its optimality}

As mentioned in the previous section, a key ingredient to achieving better convergence properties is considering gradient-based algorithms.
Therefore, we require the subgradient of the quantile loss.
It is defined in~\eqref{ql_gradient}.
\begin{align}
    {\QL}_p'(q,y)
     & = \mathbb{1}\{y < q\} -p \label{ql_gradient} .
\end{align}
Unfortunately, the quantile loss is not exp-concave. Thus an application of the EWAG does not guarantee optimal convergence properties. In contrast, the CRPS is exp-concave for random variables with bounded support~\citep[p. 12]{korotin2019integral}. Thus, when considering learning algorithms with the structure~\eqref{eq_comb_crps_cw_prop} we receive that the corresponding EWA algorithms satisfy selection optimality~\eqref{eq_optp_select} and for gradient based algorithms additionally convex optimality~\eqref{eq_optp_conv}. However, those optimalities hold with respect to the risks $\what{\mathcal{R}}^{\CRPS}_{t,\min}$ and $\what{\mathcal{R}}^{\CRPS}_{t,\conv}$, and suffer from the lack of optimality~\eqref{eq_risk_ql_crps_expert} and~\eqref{eq_risk_ql_crps_convex}.

For optimality of pointwise CRPS learning algorithms with respect to $\overline{\what{\mathcal{R}}}^{\QL}_{t,\min}$ and $\overline{\what{\mathcal{R}}}^{\QL}_{t,\conv}$ we have to consider the BOAG algorithm~\eqref{algo:boa_general_step} and analyze the quantile loss further concerning~\eqref{eq_boa_opt_select} and~\eqref{eq_boa_opt_conv}. By~\eqref{ql_gradient} it is clear that the quantile loss ${\QL}_p$ is convex. Thus, the fully adaptive BOAG satisfies~\eqref{eq_boa_opt_conv} with respect to ${\QL}_p$. However, if~\eqref{eq_boa_opt_select} holds is not obvious. As piecewise linear function ${\QL}_p$ is Lipschitz continuous. To check if weak exp-concavity holds for ${\QL}_p$ we have to investigate the conditional quantile risk $\mathcal{Q}_{t,p}(x) = \mathbb{E}[ {\QL}_p(x, Y_t) | \mathcal{F}_{t-1}]$. The convexity properties of $\mathcal{Q}_{t,p}$ depend on the conditional distribution $Y_t|\mathcal{F}_{t-1}$:
\begin{proposition}\label{proposition_qrisk}
    Let $Y$ be a univariate random variable with (Radon-Nikodym) $\nu$-density $f$ with $\E[|Y|]<\infty$, then for the second subderivative of the quantile risk
    $\mathcal{Q}_p(x) = \mathbb{E}[ {\QL}_p(x, Y) ]$
    of $Y$ it holds for all $p\in(0,1)$ that
    $$\mathcal{Q}_p'' = f.$$
    Additionally, if $f$ is a continuous Lebesgue-density with $f\geq\gamma>0$ for some constant $\gamma>0$ on its support $\text{spt}(f)$ then
    is $\mathcal{Q}_p$ is $\gamma$-strongly convex.
\end{proposition}
This quantile risk characterization is interesting as it provides a results on the convexity behavior independent of $p$. Note that the requirements for $\gamma$-strong convexity are somewhat restrictive, e.g. $Y$ must be bounded. But, we will use the proposition for $Y_t|\mathcal{F}_{t-1}$, and $Y_t|\mathcal{F}_{t-1}$ may be bounded even though $Y_t$ is unbounded.

Proposition~\ref{proposition_qrisk} is valuable for us because strong convexity with $\beta=1$ implies weak exp-concavity~\citep[p. 4]{gaillard2018efficient}.
Hence, if it holds for $\mathcal{Q}_p$, it implies almost optimality of the fully adaptive gradient-based BOA with multiple learning rates for the quantile loss $\QL_{p}$.
This is sufficient to deduce
the selection and convexity optimality~\eqref{eq_optp_select} and~\eqref{eq_optp_conv}
of pointwise CRPS learning with respect to
$\overline{\what{\mathcal{R}}}^{\QL}_{t,\min}$
and
$\overline{\what{\mathcal{R}}}^{\QL}_{t,\conv}$, despite the $\log(\log(t))$ term:
\begin{theorem} \label{thm_boag}
    The gradient based fully adaptive Bernstein online aggregation (BOAG)~\eqref{algo:boa_general_step} applied pointwise for all $p\in(0,1)$ as in~\eqref{eq_forecast_F_def} on the quantile loss ${\QL}_p$ provides an aggregation procedure
    satisfying~\eqref{eq_boa_opt_conv}
    with  minimal $\CRPS$ given by
    $\what{\mathcal{R}}_{t,\conv} =
        \overline{\what{\mathcal{R}}}^{\QL}_{t,\conv}$
    if $V < \infty$.
    If $Y_t|\mathcal{F}_{t-1}$ has a pdf $f_t$ satifying $f_t>\gamma >0$ on its support $\spt(f_t)$ then~\eqref{eq_boa_opt_select} holds with $\beta=1$ and
    $\what{\mathcal{R}}_{t,\min} = \overline{\what{\mathcal{R}}}^{\QL}_{t,\min}$.
\end{theorem}
The assumption for the convex aggregation property is very mild; we only require the convergence of the variance terms $\frac{1}{t}V_{t,k} \to V_k<V$. This is satisfied if the instantaneous regret sequence $r_{t,k}$ is stationary. In contrast, the requirements for the selection property are more restrictive as the assumption on the density of $Y_t|\mathcal{F}_{t-1}$ implies that it has to be bounded. We remark that the optimality of Theorem~\ref{thm_boag} holds even though quantile crossing might happen, see Subsection~\eqref{subsec_quantile_crossing}. Of course, Theorem~\ref{thm_boag} is more of a theoretical value, as we can never evaluate the quantile loss for infinitely many probabilities $p\in(0,1)$ in practice. As mentioned, a common reporting situation considers a dense equidistant grid of probabilities $\bsPP$.
So quantile forecasts $\what{F}_{t,k}^{-1}(p)$ for $p\in \bsPP$ are revealed every time $t$ for each expert $k$ and provide a suitable approximation of the CRPS, see~\eqref{eq_crps_approx}.

\subsection{Relationship to Bayesian model averaging}

Bayesian model averaging (BMA) is a popular combination method in the econometric literature~\citep{fragoso2018bayesian, aastveit2019evolution}. \citet{raftery2005using} introduced how BMA can be used for forecasting.
On the first view, there is no clear link between the learning theory explained in the previous subsection and BMA.

BMA under uninformative priors can be approximated by~\eqref{eq_ewa_general}~while replacing the regret $\bsR_{t}$ by the negative Bayesian information criterion (BIC) and choosing $\eta=\frac{1}{2}$
\begin{align}
    \bsw_{t}^{\text{BMA}} & =
    \softmax\left(-\frac{1}{2} {\bm{\BIC}}_t\right)
    \label{eq_bma_general}
\end{align}
with $K$-dimensional BIC-vector ${\bm{\BIC}}_t = (\BIC_{t,1},\ldots, \BIC_{t,K})'$~\citep{hansen2008least}. Thus, exchanging the regret which measures the past predictive performance by a measure for model validility leads directly to BMA.  This BMA approximation procedure has been used in different combination tasks~\citep{cheng2015forecasting, maciejowska2020pca}. The BIC values of the $K$ expert models could be computed in a batch setting, or in an online setting, e.g. using RLS or Kalman filtering.

However, as we minimize the CRPS or the quantile loss, we should adjust the BIC definition in~\eqref{eq_bma_general} with respect to the considered loss. To do so, we consider a definition of the BIC based on the quantile loss as used by~\citet{lu2015jackknife} and \citet{wang2019jackknife}. For each probability $p\in(0,1)$ we define
\begin{equation}
    {\BIC}_{t,k}(p)= 2  \log \left( \frac{1}{T}\sum_{i=t-T+1}^{t} \rho_p(\varepsilon_{t,k}) \right) +  \log(T) \frac{\text{DF}_{t,k}(p)}{T}
    \label{eq_bic}
\end{equation}
where $T$ is the sample size (calibration window length) of the considered model, $\text{DF}_{t,k}$ the (effective) degrees of freedom in the model and $\varepsilon_{t,k}(p)$ the residuals. When replacing $\rho_p(\varepsilon_{t,k}(p))$ by $\varepsilon_{t,k}^2$ this turns to the standard least square BIC, which corresponds to the likelihood-based BIC under the normality assumption. We see that BMA requires some in-sample model to evaluate residuals, while this is not required for online learning algorithms like EWA or BOA. Still, the non-parametric definition~\eqref{eq_bic} also allows for applications where the prediction task is performed by quantile regression (on the dense grid $\bsPP$). The key difference between the BMA approximation as used in~\eqref{eq_bma_general} and EWA~\eqref{eq_ewa_general} is that EWA evaluates the quantile loss of the past forecasts at time $t$ whereas BMA evaluates the quantile loss of the most recent model.

From the applied point of view, nothing stops us from replacing the BIC in~\eqref{eq_bma_general} by another information criterion. Alternatively, the factor $\frac{1}{2}$ can be replaced by another learning rate $\eta$:
\begin{align}
    \bsw_{t}^{\text{BMA}}(\eta) &
    = \softmax\left(-\eta {\bm{\BIC}}_t\right)
    \label{eq_bma_eta}
\end{align}
In this case, we may search for optimal values $\eta$ based on data-driven methods, e.g., on past performance. We recommend this procedure, as the BMA approximation~\eqref{eq_bma_general} is derived for the likelihood-based BIC, whereas the utilization of the quantile loss-based BIC~\eqref{eq_bic} might require a different scaling.

\section{CRPS learning extensions and implementation remarks}\label{sec_extentions}

This section briefly covers extensions to the considered CRPS learning algorithms and provides some implementation remarks. The extensions hold for online learning methods in the same way as for standard batch procedures based on quantile regression. First, we discuss a penalized smoothing procedure for the weights in CRPS learning algorithms. Next, we discuss forgetting in learning algorithms and extensions by the three simple shrinkage operators: fixed shares, soft-thresholding, and hard-thresholding. Note that the extensions are mainly motivated by intuition, application in similar situations in statistics and machine learning, and less by theoretical results. We partially back up the extension features by simulation and empirical studies afterward.

\subsection{Penalized smoothing}\label{subsec_smooth}

Our weight functions $w_{t,k}$ are modeled as linear combinations of basis functions as explained in Subsection~\ref{subsec_est_weight_functions}. The corresponding function space, resp. the basis function $\bsvarphi=(\varphi_1,\ldots, \varphi_L)$ restrict the shape of possible weight function estimates. If we consider B-Splines as a basis, it is relatively straightforward to extend this standard smoothing method to penalized splines (P-Splines), at least in batch learning settings~\citep{wang2011smoothing}. The idea of P-Splines is to regularize the roughness of the fitted curve. This is particularly useful if a high number of basis functions is considered. For example, in the extreme case, the number of basis functions $L$ equals the size of the probability grid $\bsPP$, as in situation~\eqref{eq_qr}.

Theoretically, we can extend~\eqref{eq_qr_basis} by a roughness penalty term:
\begin{align}
     & \bsbeta_{t}^{\text{SmCRPS}}     \label{eq_qr_basis_pen}
    \\
     & = \argmin_{ \bsbeta \in {\R}^{K\times L}} \sum_{i=t-T+1}^t  \int_{0}^1 {\rho}_p \left( Y_i - \sum_{k=1}^K \sum_{l=1}^L \beta_{k,l}  \varphi_l(p)\what{F}^{-1}_{i,k}(p) \right) \,d p
    + \text{SmPen}_{\lambda}(\bsbeta) . \nonumber
\end{align}
A standard choice for $\text{SmPen}_{\lambda}(\bsbeta)$ would penalize the weight functions $ \bsbeta_{t,k}' \bsvarphi$ based on their second derivative, as applied in generalized additive models (GAM)~\citep{wood2017generalized}.
However, this increases the difficulty of the optimization problem even further and makes it impractical for applications.

Fortunately, there is a suitable penalized smoothing alternative that can be applied to $\bsbeta_{t,k}' \bsvarphi$ in batch and online settings, based on~\eqref{eq_qr_basis_p}. The idea is to smooth the weight functions after the estimation (i.e., updating) step. Therefore, we take another set of bounded basis functions $\bspsi=(\psi_1,\ldots, \psi_H)$ on $(0,1)$ that we will use for penalized smoothing.

Then we can represent the weights as
$$w_{t,k} = \sum_{h=1}^H b_{k,h} \psi_h = \bsb_k'\bspsi $$
with parameter vectors $b_k$. We estimate $\bsb_k=(b_{k,1},\ldots,b_{k,H})$ by penalized $L_1$- and $L_2$-smoothing which minimizes for each $k$ given $\bsbeta_{t,k}$ by
\begin{equation}
    \| \bsbeta_{t,k}' \bsvarphi  - \bsb_k' \bspsi  \|^2_2 + \lambda ( \alpha \| \mathcal{D}_{1}  (\bsb_k' \bspsi)  \|^2_2 + (1-\alpha)\| \mathcal{D}_{2}  (\bsb_k' \bspsi)  \|^2_2 )
    \label{eq_function_smooth}
\end{equation}
with differential operator $\mathcal{D}_d$ of order $d$. The differential order characterizes the type of smoothing penalty. The parameter $\lambda\geq 0$ chracterises the roughness penalty and  $\alpha$ mixes between the penalty in the first and second derivative.
Typically, $\alpha=0$ is considered along with cubic B-Splines to penalize for roughness~\citep{wang2011smoothing, wood2017generalized}.
However, we prefer using $\alpha>0$ here. As for $\alpha>0$, we are smoothing towards constant weights over $\bsPP$ for $\lambda\to \infty$.
and not towards a linear relationship between weights and probabilities as for $\alpha=0$.
From our point of view, there is no clear argument
in CRPS learning that supports shrinkage towards a linear relationship. In contrast, shrinkage towards a constant brings us directly to the non-pointwise CRPS-learning theory of constant weight functions. However, this does not mean
that the penalized smoothing situation with $\lambda\to \infty$ gives the same result as the plain CRPS learning approach with $\bsvarphi = \varphi_1 \equiv 1$. The first simulation study (see~\ref{sim1}) gives more insights.
Our simulation experiments show that the choice $\alpha=1/2$ leads to suitable results in practice.

In applications, we are considering the function based optimization problem~\eqref{eq_function_smooth} only on the equidistant probability grid $\bsPP=(p_1,\dots,p_M)$ of size $M$.
If we consider a B-Spline basis functions $\bspsi$,
then an explicit solution based on ordinary least squares exists for~\eqref{eq_function_smooth}.
If the basis functions have additionally equidistant knots, this explicit solution has a simple ridge regression representation. The smoothed weights matrix $\bsw_t$ is then given by
\begin{equation}
    \bsw_{t}(\bsPP) = \bsB(\bsB'\bsB+ \lambda (\alpha \bsD_1'\bsD_1 + (1-\alpha) \bsD_2'\bsD_2))^{-1} \bsB' \bsvarphi (\bsPP) \bsbeta_{t}
    \label{eq_smoothing_solution_w}
\end{equation}
with matrix $\bsB= \bspsi(\bsPP)$, penalty matrix
$\bsD_d= \Delta^d \bsI$ and $\Delta$ denotes the difference operator.

At first glance, it seems odd that we consider another basis $\bspsi$ for the penalized smoothing approach than we used for the original basis $\bsvarphi$ for the weight function representation. However, we do so mainly to emphasize that these can be regarded as two separate problems.
Indeed, for applications we recommend to use the same basis $\bspsi = \bsvarphi$.
The latter shows excellent empirical performance and is easier to implement and communicate.

A penalized smoothing algorithm for CRPS learning using BOA~\eqref{algo:boa_general_step} is presented in Algorithm~\ref{algo:boag_smooth} where $\bsQL_{\bsPP} = ({\QL}_{p_1},\ldots,{\QL}_{p_M})$ and $\bssoftmax(\bsx) = (\softmax(\bsx_1), \ldots, \softmax(\bsx_L))$. The attentive reader will notice that the algorithm uses the BOA update step not for the weights vectors $\bsw_{t}$ as in~\eqref{algo:boa_general_step} but for the $\bsbeta_{t} = (\beta_{t,l,k})_{l,k} \in \R^{L\times K}$ matrices. However, the algorithm has the same objective as in~\eqref{eq_qr_basis_p}. Thus, $\beta_{t,l,k}\varphi_l$ may be regarded as weight functions for the expert predictions. The special case $\bsvarphi = \varphi_1 \equiv 1$ leads to constant weight functions with respect to the CRPS. For $\varphi_i=\mathbb{1}{\{p_i\}}$ on $\bsPP$ (e.g. choosing linear B-Splines with equidistant knots of distance $p_i - p_{i-1}$) $\bsbeta_{t,l}$ coincides with the pointwise application of BOA~\eqref{algo:boa_general_step} with respect to $\QL_p$ on $\bsPP$. In this case Algorithm~\ref{algo:boag_smooth} satiesfies Theorem~\ref{thm_boag} on $\bsPP$ as well.

\begin{algorithm}[ht!]
    \caption{\label{algo:boag_smooth} P-Smoothed CRPS Bernstein Online Aggregation}
    \begin{algorithmic}
        \State \textbf{input:}
        3-dimensional array of expert predictions $(\what{X}_{t,p,k})_{t,p,k}$ and vector of prediction targets $Y_t$ for $t=1,\ldots, T$, $p\in \bsPP = (p_1,\ldots,p_M)$, $k=1,\ldots, K$
        \State \textbf{initialize:} $\bsw_0=(w_{0,1},\ldots, w_{0,K})$,
        $\lambda\geq 0$, $M\times L$-dimensional
        P-Spline matrices $\bsB$, $\bsD_1$ and $\bsD_2$ as in~\eqref{eq_smoothing_solution_w} for the basis $\bsvarphi = (\varphi_1,\ldots,\varphi_L)$ on $\bsPP$,
        $\bsE_{0} = \bsV_0 = \bsR_0 = \bsnull$
        \State $\bsbeta_0 = \bsB^{\text{pinv}}\bsw_0(\bsPP)$
        \For{$t$ in $1,\ldots, T$}
        \State $\wtilde{\bsX}_{t} = \text{Sort}\left( \bsw_{t-1}'(\bsPP) \what{\bsX}_{t} \right)$
        \State $\bsr_{t} = \frac{L}{M} \bsB' \left({\bsQL}_{\bsPP}^{\nabla}(\wtilde{\bsX}_{t},Y_t)- {\bsQL}_{\bsPP}^{\nabla}(\what{\bsX}_{t},Y_t)\right) $ 
        \State      $  \bsE_{t}    = \max(\bsE_{t-1}, \bsr_{t}^+ + \bsr_{t}^-)    $
        \State $ \bsV_{t}     = \bsV_{t-1} + \bsr_{t}^{ \odot 2} $
        \State $ \bseta_{t}  =\min\left( \left(-\log(\bsbeta_{0}) \odot \bsV_{t}^{\odot -1} \right)^{\odot\frac{1}{2}} , \frac{1}{2}\bsE_{t}^{\odot-1}\right)                                                                                                  $
        \State $\bsR_{t}     = \bsR_{t-1}+  \bsr_{t} \odot \left( \bsone - \bseta_{t} \odot \bsr_{t} \right)/2 + \bsE_{t} \odot \mathbb{1}\{-2\bseta_{t}\odot \bsr_{t} > 1\}  $
        \State $\bsbeta_{t}     = K \bsbeta_{0} \odot \bssoftmax\left( -  \bseta_{t} \odot \bsR_{t} + \log( \bseta_{t}) \right)$
        \State $\bsw_{t}(\bsPP) = \bsB(\bsB'\bsB+ \lambda (\alpha \bsD_1'\bsD_1 + (1-\alpha) \bsD_2'\bsD_2))^{-1} \bsB' \bsB \bsbeta_{t}$
        \EndFor
    \end{algorithmic}
\end{algorithm}

Moreover, we want to remark that
the initial weight matrix $\bsw_0$ has to be transformed in the
$\bsvarphi$-space to receive an initial $\bsbeta_0$.
Thus, we have to solve $\bsB\bsbeta_0 = \bsw_0$. In general $\bsB$ is not invertible, but a simple application of the pseudoinverse yielding $\bsbeta_0 = \bsB^{\text{pinv}}\bsw_0$ is completely sufficient for applications.

In practice, it is crucial to choose an adequate tuning parameter $\lambda$. Therefore, we recommend choosing $\lambda$ based on the past performance of the forecaster as described in Subsection~\ref{subsec:impl}.
Note that the application of smoothing in~\eqref{eq_smoothing_solution_w} may look suboptimal as we optimize $\beta$ with respect to the squared loss, not to the CRPS as in~\eqref{eq_qr_basis_pen}. However, we can still optimize $\lambda$ with respect to the target CRPS.

\subsection{Forgetting}

Suppose we expect structural breaks in either the prediction target or the expert forecasts. In that case, it makes sense to reduce the weight of past information compared to recent information.
Therefore, algorithms are sometimes extended by a forgetting factor (also called discounting factor). Exponential forgetting is usually preferred in online learning due to the recursive structure.
The regret $R_{t,k}$~\eqref{eq_regret} can be written recursively as $ R_{t,k}  = R_{t-1,k} + \ell(\wtilde{F}_{t},Y_t) - \ell(\what{F}_{t,k},Y_t)$.
The regret with forgetting factor $\xi\in[0,1]$ discounts the regret $R_{t-1,k}$ in the definition above
\begin{align}
    R_{t,k}(\xi) & =  (1-\xi) R_{t-1,k} + \ell(\wtilde{F}_{t},Y_t) - \ell(\what{F}_{t,k},Y_t)\label{eq_regret_forget}
\end{align}
where $\xi=0$ corresponds to no forgetting.
In practice, optimal $\xi$ values are usually close to $0$. 
We want to highlight that in online procedures forgetting should be applied
to all recursive hidden state variables. In case of
the fully adaptive BOA~\eqref{algo:boa_general_step} these
are the regret vector $\bsR_t$, the range vector $\bsE_t$ and the variance vector $\bsV_t$.
For the quantile regression setting exponential forgetting corresponds to the weighted regression problem
\begin{align*}
    \bsbeta_{t}^{\bsvarphi\text{-QR}}(\xi)
    = \argmin_{ \bsbeta \in {\R}^{K\times L}} \sum_{i=t-T+1}^t  \sum_{p\in\bsPP} (1-\xi)^{t-i} \rho_p \left( Y_i -  \sum_{l=1}^L \sum_{k=1}^K \beta_{k,l} \varphi_l(p) \what{F}^{-1}_{i,k}(p) \right).
\end{align*}
Even though exponential forgetting is popular in online learning, we cannot expect to preserve any asymptotic guarantees because the exponential function decays too fast \citet[p.32]{cesa2006prediction}. However, polynomial discounts can be considered to keep convergence guarantees.
However, exponential forgetting dominates the academic literature on online learning, see, e.g., \citet{guo2018online, messner2019online,ziel2021smoothed}.

\subsection{Shrinkage operators}\label{subsec_fixed_share}

In statistical learning theory, shrinkage operators help to reduce overfitting problems by shrinking a solution.
The P-Spline smoothing can be interpreted as a shrinkage operator applied to the weight functions $\bsw_t = \bsB \bsbeta_{t}$.
However, simple shrinkage operators can also be applied to $\bsbeta_{t}$.

We want to briefly discuss three shrinkage operators that are sometimes used in online learning settings: the fixed share operator $\mathcal{F}$, the soft-thresholding operator $\mathcal{S}$, and the hard-thresholding operator $\mathcal{H}$.
They are defined as
\begin{align}
    \mathcal{F}(x;\phi)   & = \phi/K  + (1-\phi) x ,       \\
    \mathcal{S}(x;\nu)    & = \sign(x)||x|-\nu| ,          \\
    \mathcal{H}(x;\kappa) & = x  \mathbb{1}\{|x|>\kappa \}
\end{align}
for some $\phi\in[0,1]$, $\nu\geq0$ and $\kappa\geq0$.

The fixed share operator shrinks towards the naive combination. The naive combination is preferable if there is no prior information on the expected performance of the experts available. For some shrinkage problems, even theoretical guarantees for improvements exist~\citep{tu2011markowitz, cesa2012mirror}. In the online learning context it is used, e.g., in \citet{cesa2012mirror, gonzales2021new}.

The thresholding operators $\mathcal{S}$ and $\mathcal{H}$ are also sometimes considered in online learning contexts~\citep{dalalyan2012sharp, gaillard2017sparse}.
Thresholding is a technique to receive sparse solutions. Both appear in several situations for specific linear model estimators. Most notably, soft-thresholding is the key operator in the coordinate descent algorithm for estimating the lasso~\citep{friedman2007pathwise}.
Let us remark that the application of any of the threshold operators potentially violates affinity constraints (incl. the convexity constraint). In those situations, projections to the desired solution space should be applied.

\subsection{Implementation Remarks}\label{subsec:impl}

Algorithm~\ref{algo:boag_smooth} is implemented along with
other algorithms (e.g. EWA(G) \citet{gaillard2015forecasting}, ML-Poly(G) \citet{gaillard2014second}) in
the freely available \texttt{R}-package \texttt{profoc}~\citep{profoc_package}.
The functions are implemented using RcppArmadillo to limit computation time~\citet{eddelbuettel2014rcpparmadillo}, sparse matrix algebra is utilized since the basis matrices $\bsB$ are usually sparse, and, in the case of P-Smoothing, the hat-matrix in~\eqref{eq_smoothing_solution_w} will be pre-computed as it stays constant during learning. The \texttt{profoc} package also supports batch learners for comparison purposes.

Next to the P-Spline smoothing parameters ($\lambda$ and $\alpha$), the forget parameter ($\xi$) and shrinkage operator parameters ($\phi$, $\nu$, and $\kappa$), \texttt{profoc} offers even more tuning parameters, e.g., for choosing non-equidistant grids for the B-Spline basis and for choosing basis functions of different degrees.

We recommend choosing the parameter combination based on the past cumulative performance of the forecaster. This approach is also implemented in \texttt{profoc}. It considers all possible parameter combinations at each step. Finally, it selects the combination weights $\bsw_{t}(\bsPP)$ corresponding to the parameter combination with the lowest cumulative loss up to $t-1$.

%% file: 04_simulation_study.tex
We present three simulation studies in this chapter. The first compares different smoothing methods and discusses the role of $\alpha$. The second compares the performance of different online learning algorithms while still using the simple data generating process (DGP) used for the first simulation. After that, a third study presents the application of an exponential forget (see~\eqref{eq_regret_forget}) in an environment of changing optimal weights.
Each study presents the mean of 1000 repetitions. We always use $K=2$ experts as it allows us to compute the optimal weight functions analytically. We initialize the weights uniformly: $w_{0,k} \equiv 1/K$ and use percentiles as probability grid $\bsPP=(0.01,\dots,0.99)$.

\subsection{Smoothing Methods}\label{sim1}

The DGP for the first two studies is kept simple. We consider two experts with constant predictions over time. The true process is standard normal:
\begin{align}
    Y_t & \sim \mathcal{N}(0,\,1)     ,
    \widehat{X}_{t,1}  \sim \widehat{F}_{1}  = \mathcal{N}(-1,\,1) \text{ and }
    \widehat{X}_{t,2}  \sim \widehat{F}_{2}  = \mathcal{N}(3,\,4) .
    \label{eq:dgp_sim1}
\end{align}

That is, optimal weight functions $w_{t,k}$ do not change over time. However, they are not constant over all $p \in \bsPP$. We use this DGP to compare the following specifications of the proposed \textbf{BOAG} algorithm~\ref{algo:boag_smooth}. First, the basis $\bsvarphi$ is specified so that $\varphi_i=\mathbb{1}{\{p_i\}}$, which means that no smoothing is applied. We refer to this as the \textbf{Pointwise} solution. The first smooth solution \textbf{B-Smooth} is obtained by considering cubic B-spline basis' with equidistant knots and knot distances of $\{0.005, 0.02, 0.035, \ldots, 0.485, 0.5\}$. The knot distance is selected in each step with respect to the past performance. Further, the standard CRPS-Learning solution \textbf{B-Constant} is considered by using the constant basis $\bsvarphi = \varphi_1 \equiv 1$. Another specification  \textbf{P-Smooth} uses P-Spline smoothing with $\varphi_i = \psi_i = \mathbb{1}{\{p_i\}}$. The smoothing parameter $\lambda$ is chosen from an exponential grid $\Lambda= \{2^x|x\in \{-15,-14,\ldots,25\}\}$ based on the past performance. The last model which is considered approximates the limiting case of $\lambda\to \infty$ by using $\lambda = 2^{30}$ and is denoted by \textbf{P-Constant}. We choose the knot distance for \textbf{B-Smooth} and $\lambda$ for \textbf{P-Smooth} based on the past performance of the forecaster as described in Subsection~\ref{subsec:impl}.

\begin{figure}[h!]
    \fbox{
        \begin{subfigure}[t]{0.475\textwidth}
            \includegraphics[width=\textwidth]{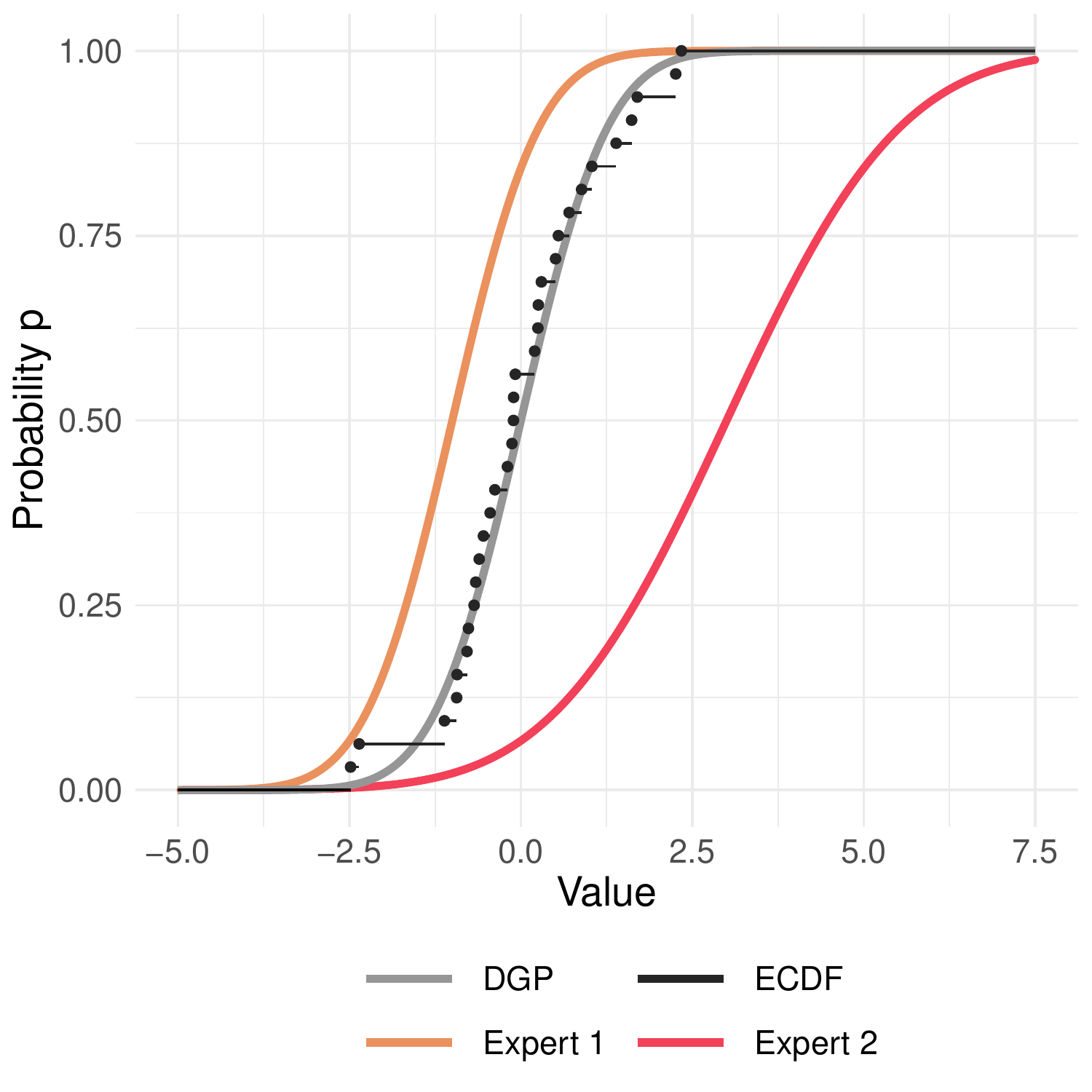}
            \caption{Two experts forecasting the true distribution (gray) of which 32 realizations are sampled (black)}
        \end{subfigure} \hfill
        \begin{subfigure}[t]{0.475\textwidth}
            \includegraphics[width=\textwidth]{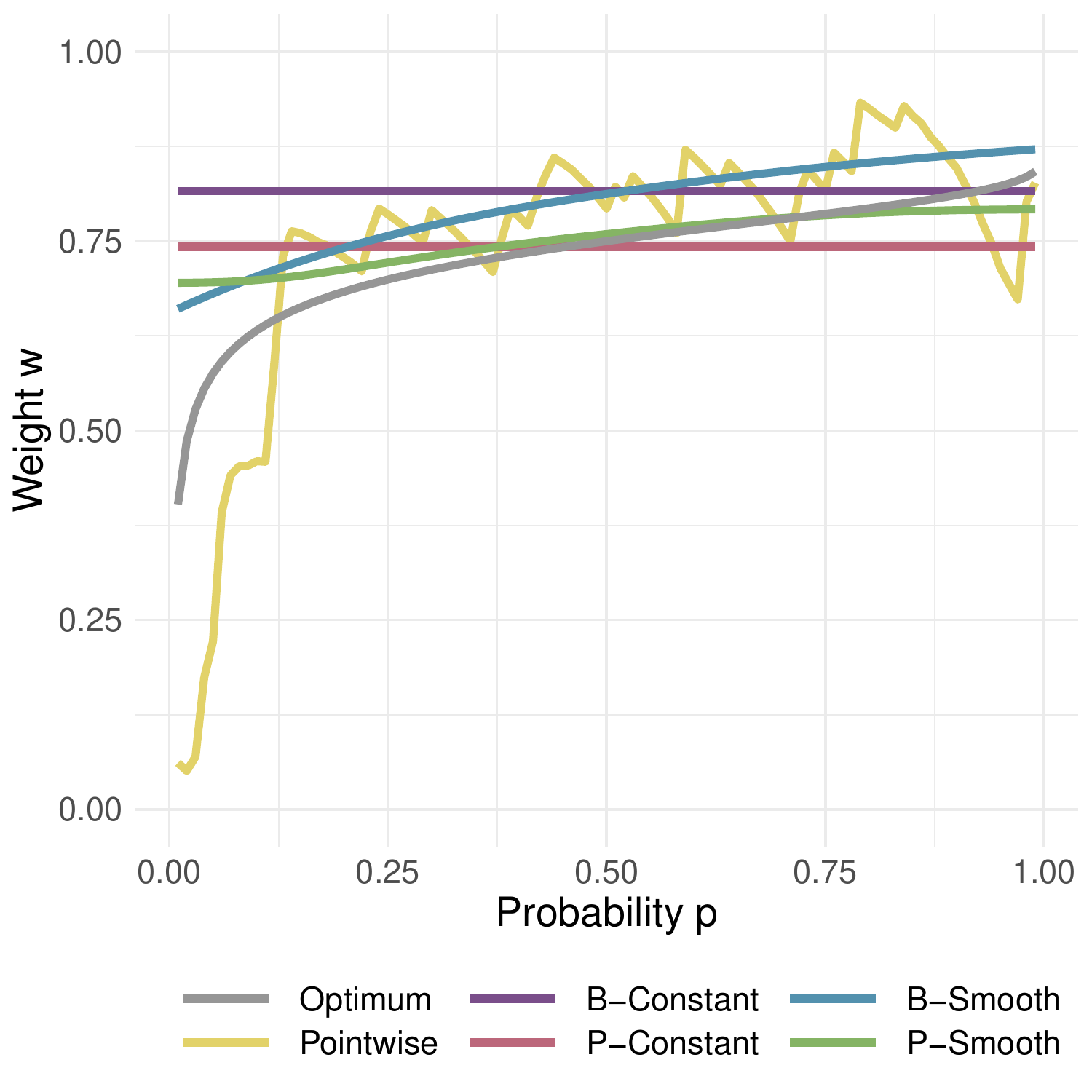}
            \caption{Weights of Expert 1 calculated by different specifications of the proposed algorithm and the theoretcal optimum (gray)}
        \end{subfigure}
    }
    \caption{Example of probabilistic combination of two expert forecasts (left side), weights of Expert 1 estimated using suitable combination algorithms and the true weighting function (right side).}
    \label{fig:mot_data}
\end{figure}

Figure~\ref{fig:mot_data} presents an example of the combination task. It shows the cumulative distribution function of the DGP and both experts as well as the empirical cumulative distribution function of a sample (left side). Further, it shows the resulting optimal weight function (gray) as well as combination weights calculated by various specifications of the proposed \textbf{BOAG} algorithm~\ref{algo:boag_smooth} (right side). Evaluating the performance using this figure is not possible since it only presents one single sample. However, we observe that the true combination weights vary over $p \in \bsPP$. That is, \textbf{B-Constant} and \textbf{P-Constant} will never convergence against the optimum.

\begin{figure}[!h]
    \centering
    \fbox{\includegraphics[width=0.975\columnwidth]{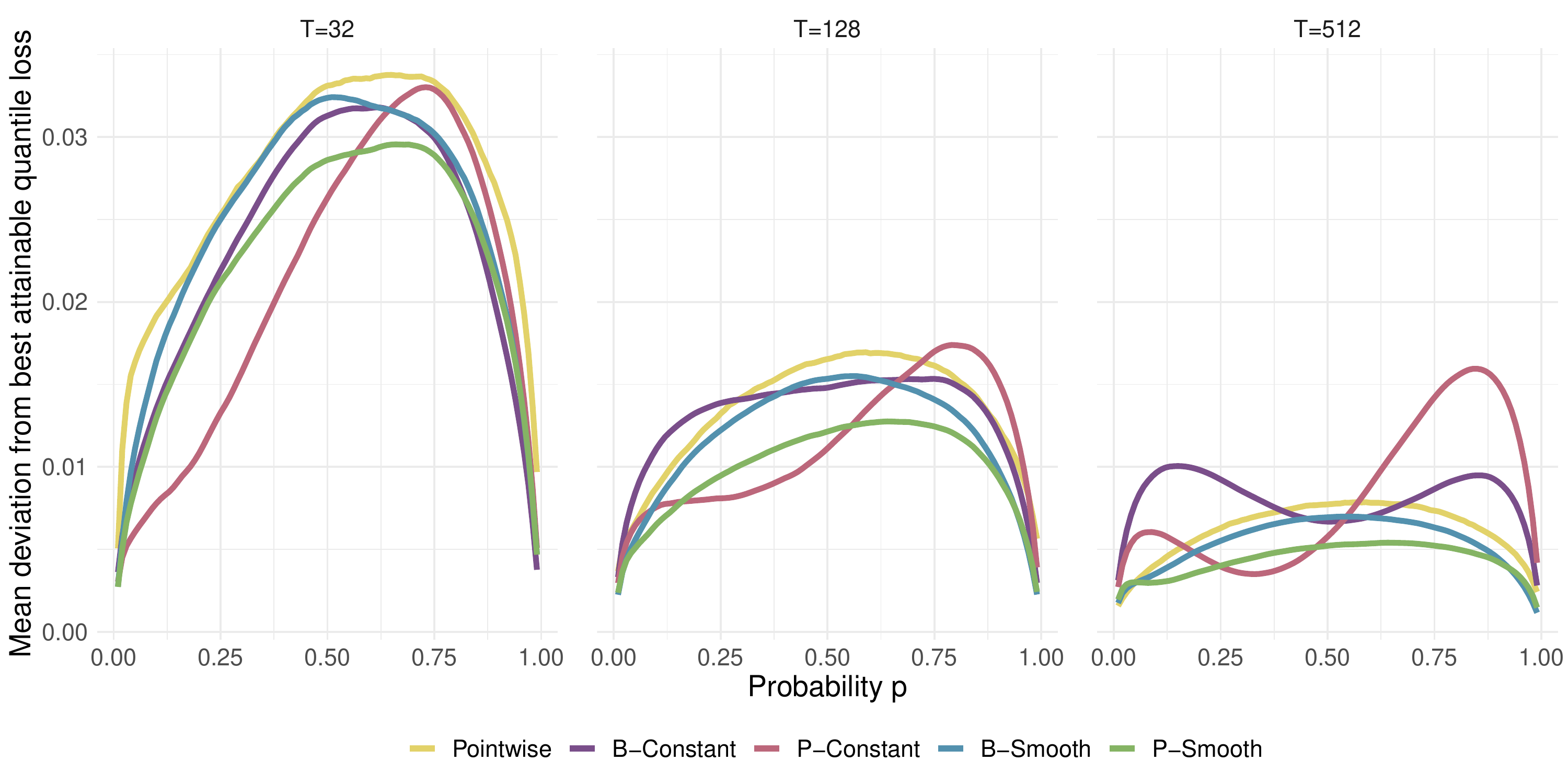}}
    \caption{Mean distance to the best attainable quantile loss for different specifications of the proposed \textbf{BOAG} algorithm~\ref{algo:boag_smooth} and different sample sizes $T$. Smaller values correspond to better performance.
    }\label{fig:sim_1_loss}
\end{figure}

Figure~\ref{fig:sim_1_loss} shows the results of this first study. It depicts the mean distance to the best possible loss with respect to $\bsPP$. The results are presented for 32, 128, and 512 observations. The distance to the best possible loss decreases as the number of observations increases, as expected. The pointwise solution outperforms both constant solutions if a sufficient amount of observations is supplied.
Furthermore, both smoothing methods improve over the pointwise solution regardless of the sample size. The \textbf{P-Smooth} specification yields the best results overall.

\begin{figure}[!h]
    \centering
    \fbox{\includegraphics[width=0.975\columnwidth]{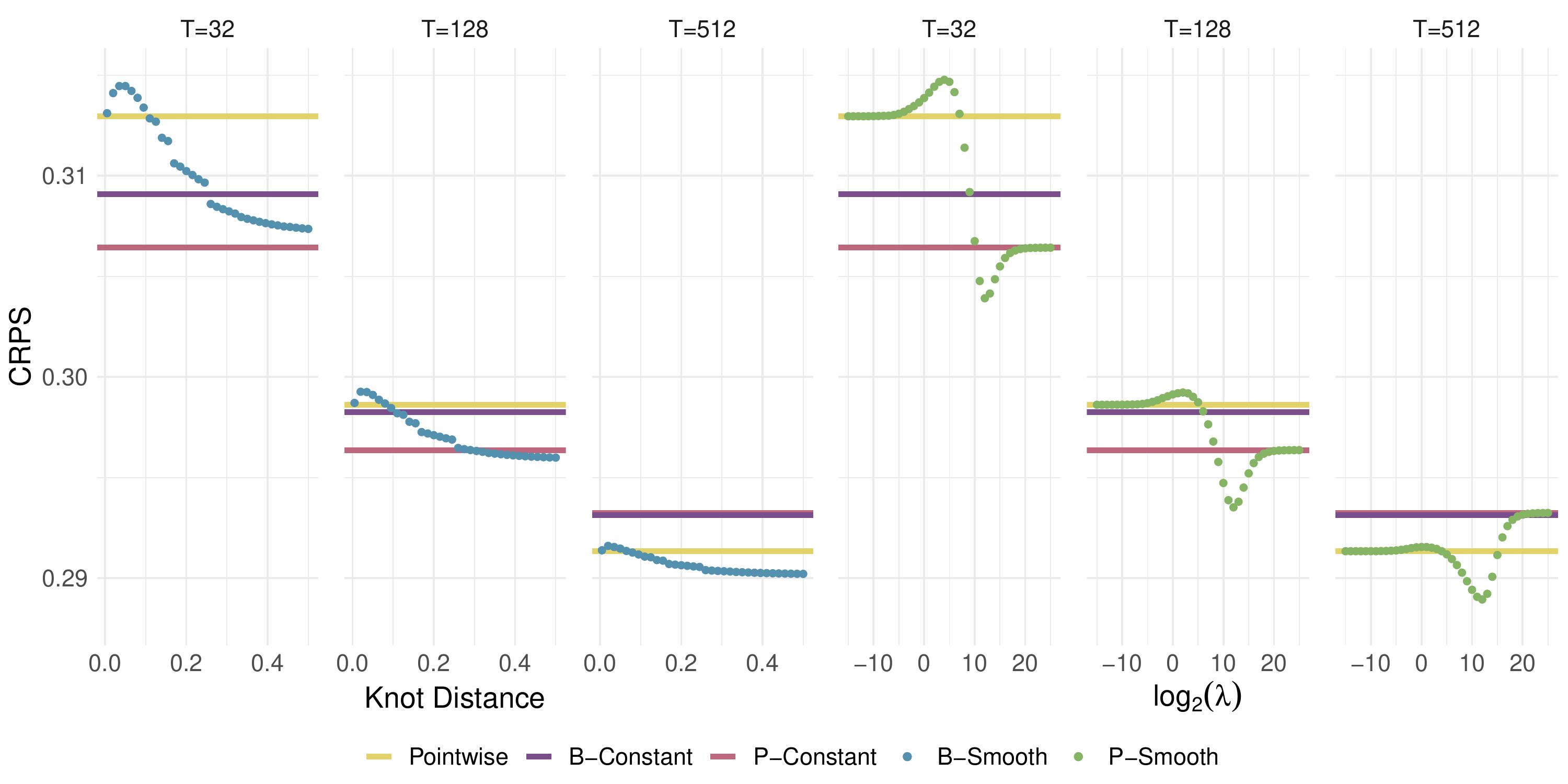}}
    \caption{CRPS values for different specifications of the proposed \textbf{BOAG} algorithm~\ref{algo:boag_smooth} with respect to $\lambda$ and the knot distance. Smaller values correspond to better performance.
    }\label{fig:lam_loss}
\end{figure}

The results of the \textbf{B-Smooth} and \textbf{P-Smooth} specifications depend on the knot distance of the basis and $\lambda \in \Lambda$ respectively. Figure~\ref{fig:lam_loss} illustrates this relation. It depicts the average CRPS for different $\lambda$ values and knot distances. It is noteworthy that \textbf{P-Constant} clearly outperforms \textbf{B-Constant} for small $T$ while for large $T$ both specifications coincide with regards to their performance. Note that \textbf{B-Constant} applies a single learning rate at all quantiles while \textbf{P-Constant} assigns learning rates pointwise for each quantile. This shows that pointwise evaluation procedures potentially outperform constant ones even if both assign constant weights in the end.

The left hand side of Figure~\ref{fig:lam_loss} shows that for \textbf{B-Smooth} a large knot distance (i.e, less knots) is preferable. This is caused by the DGP that produces smooth weight functions. Evaluating the \textbf{P-Smooth} solution with respect to $\lambda$ shows that $\lambda = 2^{12}$ yields the lowest CRPS. Figure~\ref{fig:lam_loss} also shows that \textbf{P-Smooth} yields better results than \textbf{B-Smooth} which was already apparent in Figure~\ref{fig:sim_1_loss}.

The \textbf{P-Smooth} results presented in Figures~\ref{fig:sim_1_loss} and~\ref{fig:lam_loss} were obtained using $\alpha = 0.5$. Figure~\ref{fig:crps_vs_ndiff} shows that all $\alpha > 0$ yield similar results. Therefore, setting $\alpha = 0.5$ is reasonable as it ensures penalization with respect to slope and roughness.

\begin{figure}[!h]
    \centering
    \fbox{\includegraphics[width=0.975\columnwidth]{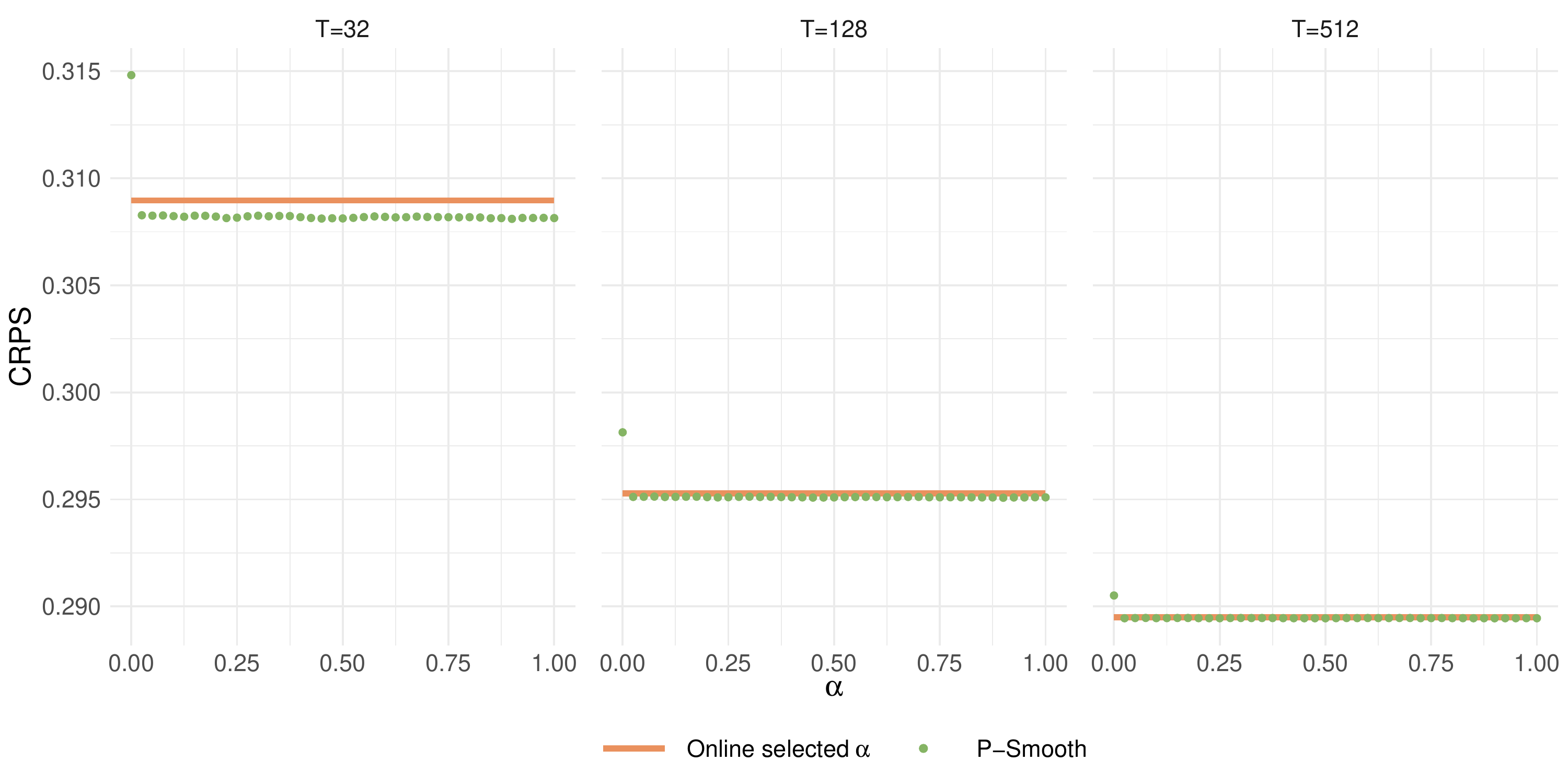}}
    \caption{CRPS values of the \textbf{P-Smooth} specification of the proposed \textbf{BOAG} algorithm~\ref{algo:boag_smooth} with respect $\alpha$ (see~\eqref{eq_function_smooth}). Smaller values correspond to better performance.
    }\label{fig:crps_vs_ndiff}
\end{figure}

\subsection{Popular learning algorithms compared}

This second simulation study uses the same DGP~\eqref{eq:dgp_sim1} to compare the proposed \textbf{BOAG}~\eqref{algo:boag_smooth} with other popular learning algorithms. In particular, we compare to \textbf{EWAG}\footnote{We consider an $\eta$-grid of $\mathcal{E}= \{(1-\sqrt{x}) |x\in \{0,0.05, 0.1, \ldots,0.95\}\}$.}~\eqref{eq_ewa_general} and the gradient version of the ML-Poly algorithm (\textbf{ML-PolyG}~\citet{gaillard2014second}).

\begin{figure}[h]
    \centering
    \fbox{\includegraphics[width=0.975\columnwidth]{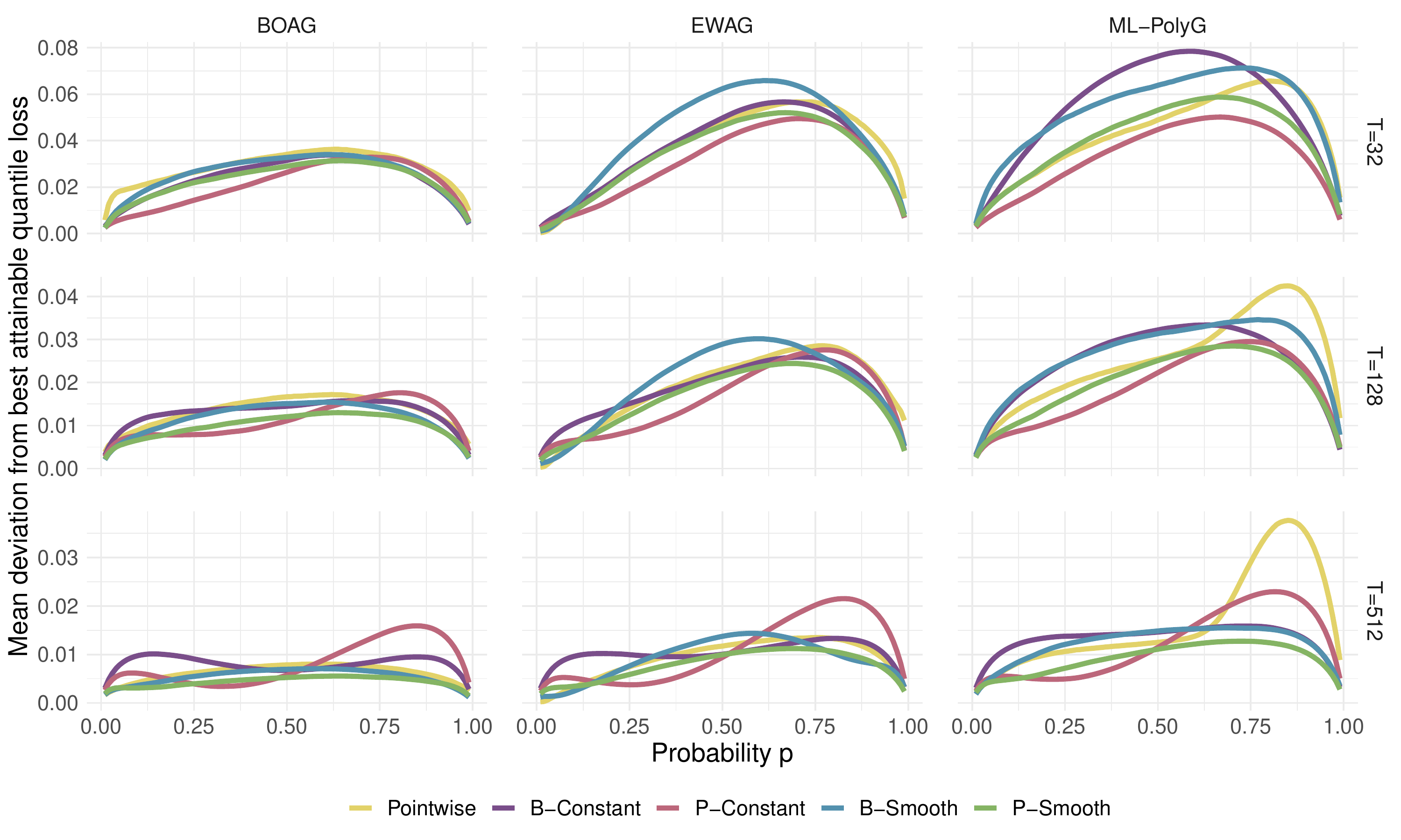}}
    \caption{Mean distance to the best attainable quantile loss for different algorithms, specifications, and sample sizes $T$. Smaller values correspond to better performance.
    }\label{fig:sim2_algos}
\end{figure}

Figure~\ref{fig:sim2_algos} presents the results of this study. The columns show the algorithms, while the rows depict different sample sizes. Analogous to the first study, we show the distance to the best possible loss. This distance decreases with increasing sample size for all three algorithms. However, the \textbf{BOAG} algorithms show lower distances for all sample sizes compared to \textbf{EWAG} and \textbf{ML-PolyG}. Additionally, we see that the constant solutions do not converge to zero since the true $\bsw_{t}$ are not constant over $\bsPP$.

To summarize, \textbf{BOAG} shows the fastest convergence indeed. Moreover, the results also support the application of the \textbf{P-Smooth} specification.

\subsection{Changing weights}

Now consider the following DGP with changing optimal weights. This DGP is set up as follows:
\begin{subequations}
    \begin{align}
        Y_t               & \sim \mathcal{N}\left(0.15 \operatorname{asinh}(\mu_t),\,1\right)
        \text{ with }
        \mu_t              = 0.99 \mu_{t-1} + \varepsilon_t   , \  \varepsilon_t \sim  \mathcal{N}(0,1),
        \\
        \widehat{X}_{t,1} & \sim      \widehat{F}_{1}  = \mathcal{N}(-1,\,1)                              \text{ and }
        \widehat{X}_{t,2}  \sim       \widehat{F}_{2}  = \mathcal{N}(3,\,4).
    \end{align}\label{eq:dgp_sim2}
\end{subequations}
We use this DGP to illustrate the importance of the forget operator~(\ref{eq_regret_forget}). Taking the entire history of regrets into account is suboptimal in settings with changing optimal weights. This is illustrated by Figure~\ref{fig:sim_forget}. It shows the optimal weights of expert two as well as weights obtained from \textbf{BOAG} with- and without an exponential forget for \textbf{Pointwise} and \textbf{P-Smooth} specifications. The models with forget are calculated by considering a $\xi$-grid of forget rates $\{2^x|x\in \{- \lceil \log_2(T)\rceil,\ldots, -2, -1\}\}$ and selecting $\xi$ in each step with respect to the past performance. The colors of Figure~\ref{fig:sim_forget} represents the weights assigned to expert two over time (x-axis) and for every $p \in \bsPP$ (y-axis).

Figure~\ref{fig:sim_forget} shows that the solutions without forget fail at precisely adjusting their weights in the long-term. Using a forget results in a more rapid weight adjustment which is favorable here. However, this Figure only presents a single example.
\begin{figure}[h]
    \centering
    \fbox{\includegraphics[width=0.975\columnwidth]{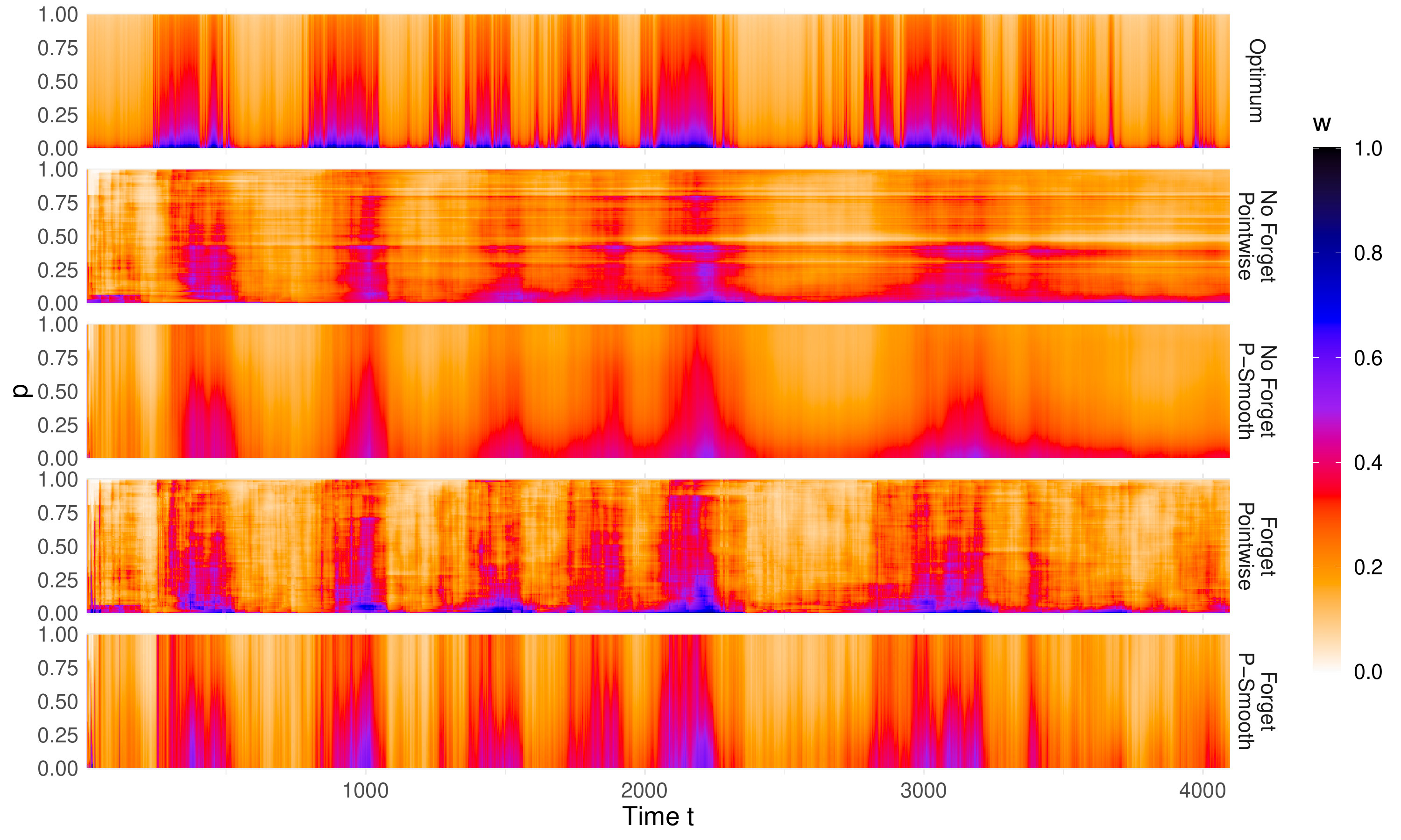}}
    \caption{Exemplary weights of expert two in an environment of weight changing over $t$ from a single optimization run. The rows present the theoretically best weights and the weights calculated using different specifications of the proposed \textbf{BOAG} algorithm~\ref{algo:boag_smooth}.}\label{fig:sim_forget}
\end{figure}
Table~\ref{tab:forget_crps} presents the average CRPS and respective standard deviations from 1000 simulation runs. The lowest CRPS was obtained by considering the forget grid as described above and using the \textbf{P-Smooth} specification.
\begin{table}[!h]
    \centering
    {\begin{tabular}{cccc}
            \hline
            \multicolumn{2}{c}{\textbf{No Forget}}                               & \multicolumn{2}{c}{\textbf{Forget}}                                                                                                                                                                                 \\                                                                                                                   \hline
            \textbf{Pointwise}                                                   & \textbf{P-Smooth}                                                     & \textbf{Pointwise}                                                   & \textbf{P-Smooth}                                                    \\
            \hline
            \cellcolor[rgb]{1,0.5,0.55} $\underset{(3.42\times10^{-6})}{0.2956}$ & \cellcolor[rgb]{1,0.518,0.5} $\underset{(3.42\times10^{-6})}{0.2953}$ & \cellcolor[rgb]{1,0.768,0.5}$\underset{(3.35\times10^{-6})}{0.2943}$ & \cellcolor[rgb]{0.5,0.9,0.5}$\underset{(3.31\times10^{-6})}{0.2930}$ \\
            \hline
        \end{tabular}}
    \caption{Mean CRPS (standard deviation) for different specifications of the proposed \textbf{BOAG} algorithm~\ref{algo:boag_smooth}. Obtained from 1000 simulation runs. Colored with respect to the CRPS.
    }\label{tab:forget_crps}
\end{table}

%% file: 05_application.tex
Now we apply the proposed CRPS learning algorithm to a forecasting task in environmental finance.
More precisely, we present a study for forecasting the prices of the European emission allowances (EUA) as traded in the
European Union Emissions Trading Scheme (EU ETS). As highlighted by \citet{segnon2017modeling}, and \citet{dutta2018modeling} ARIMA-GARCH-type models dominate the econometric finance literature, see e.g.~\citet{liu2021forecasting}. However, machine learning methods, mainly considering artificial neural networks, are utilized as well \citep{atsalakis2016using, hao2020modelling, garcia2020short}. One of the reasons is likely the rather uncertain effect of external impact factors~\citep{koop2013forecasting}.

For this illustrative application, we ignore external regressors as in \citet{benz2009modeling} and \citet{dutta2018modeling}. We consider daily EUA prices of the month-ahead future covering the full Phase III period from January 2013 until December 2020.
The price data with all $2092$ observations and the corresponding differences of log prices is illustrated in Figure~\ref{fig_eua_intro}.

\begin{figure}[htbp]
      \centering
      \fbox{\includegraphics[width=0.975\columnwidth]{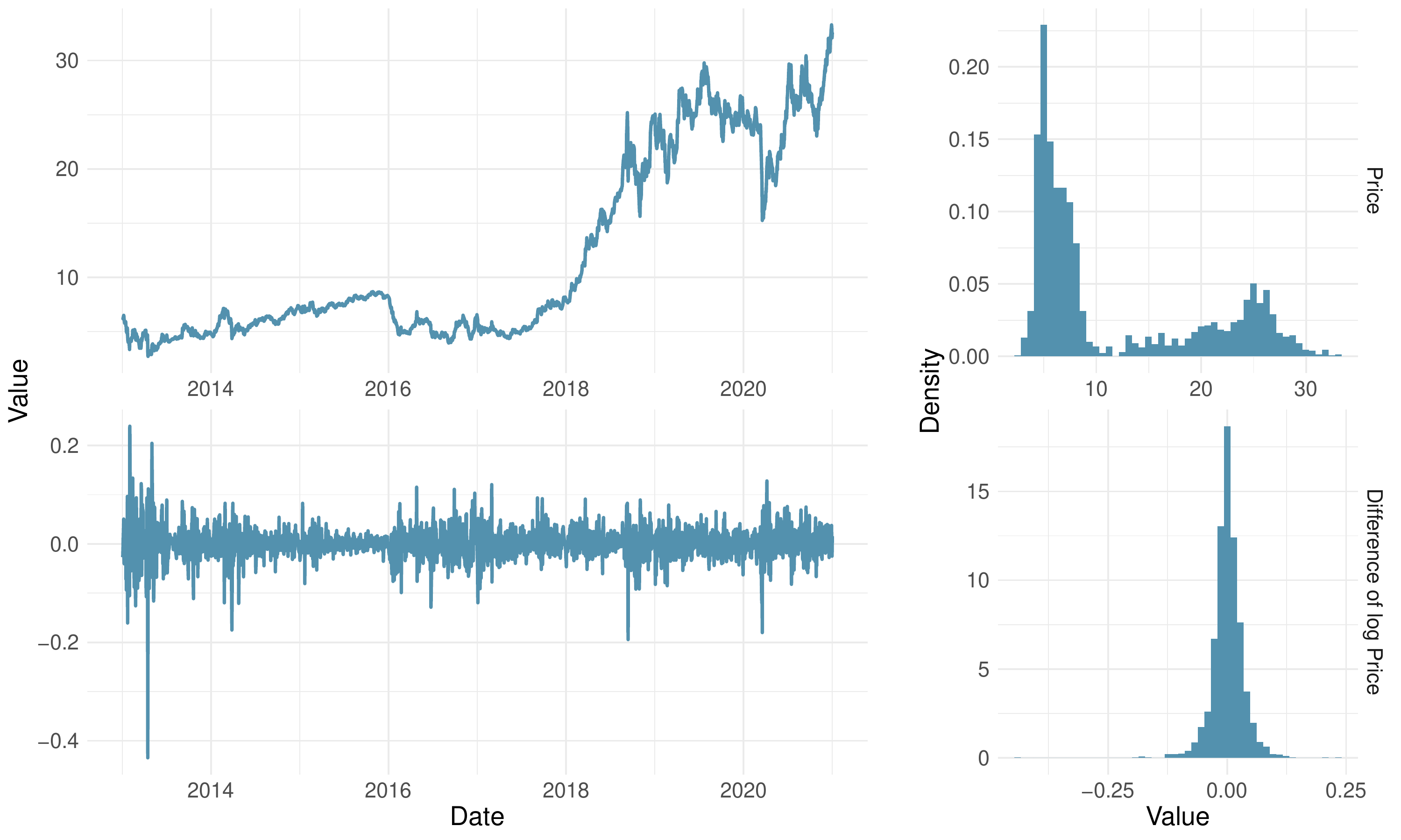}}
      \caption{Prices and differences of log Prices with corresponding histograms of the
            European Emission Allowances (EUA) during Phase III from January 2013 until December 2020}
      \label{fig_eua_intro}
\end{figure}

For the probabilistic CRPS learning study that we evaluate on $\bsPP = (0.01,\ldots, 0.99)$ we consider a selection of five plausible experts.
These are an exponential smoothing model, a quantile regression approach, and three ARIMA-GARCH-type models.
The models are selected to cover a relatively broad range of features, e.g., conditionally homoskedastic and heteroskedastic, light tails and heavy tails, stationary and instationary, parametric and non-parametric.
We consider the following models for the log-prices $Y_t$:
\begin{itemize}
      \item[i)] Simple exponential smoothing with additive errors (\textbf{ETS-ANN}):\newline
            $Y_{t} = l_{t-1} + \varepsilon_t$ with $l_t = l_{t-1} + \alpha \varepsilon_t $ and $\varepsilon_t \sim \mathcal{N}(0,\sigma^2)$.
      \item[ii)] Quantile regression (\textbf{QuantReg}):\newline
            For each $p\in \bsPP$ we assume
            $F^{-1}_{Y_t}(p) = \beta_{p,0} + \beta_{p,1} Y_{t-1} + \beta_{p,2} |Y_{t-1}-Y_{t-2}|$.
      \item[iii)] ARIMA(1,0,1)-GARCH(1,1) with Gaussian errors (\textbf{ARMA-GARCH}): \newline
            $Y_{t} = \mu + \phi(Y_{t-1}-\mu) + \theta \varepsilon_{t-1} + \varepsilon_t$ with
            $\varepsilon_t = \sigma_t Z$,
            $\sigma_t^2 = \omega + \alpha \varepsilon_{t-1}^2
                  + \beta \sigma_{t-1}^2$ and
            $Z_t \sim \mathcal{N}(0,1)$.
      \item[iv)] ARIMA(0,1,0)-I-EGARCH(1,1) with Gaussian errors (\textbf{I-EGARCH}):
            \newline
            $Y_{t} = \mu + Y_{t-1}  + \varepsilon_t$ with
            $\varepsilon_t = \sigma_t Z$,
            $\log(\sigma_t^2) = \omega + \alpha Z_{t-1}+ \gamma (|Z_{t-1}|-\mathbb{E}|Z_{t-1}|)
                  + \beta \log(\sigma_{t-1}^2)$ and
            $Z_t \sim \mathcal{N}(0,1)$.
      \item[v)] ARIMA(0,1,0)-GARCH(1,1) with student-t errors (\textbf{I-GARCHt}):\newline
            $Y_{t} = \mu + Y_{t-1}  + \varepsilon_t$ with
            $\varepsilon_t = \sigma_t Z$,
            $\sigma_t^2 = \omega + \alpha \varepsilon_{t-1}^2
                  + \beta \sigma_{t-1}^2$ and
            $Z_t \sim t(0,1, \nu)$.
\end{itemize}

We estimate the parameters of model ii) using quantile regression for each $p\in\bsPP$ separately. All the other models are estimated using maximum likelihood. We conduct a rolling window forecasting study with a window length of 250, which covers about a year. This design gives us $T=1841$ out-of-sample observations.

Concerning the combination methods, we consider uniform weighting ($w_{t,k} = 1/K$, \textbf{Naive}) the online learning methods \textbf{EWAG}, \textbf{BOAG}, \textbf{ML-PolyG} and \textbf{BMA}~\eqref{eq_bma_eta}. Additionally, we consider two quantile regression batch learners based on~\eqref{eq_qr_basis_p} with a rolling window size of 30. One of the quantile regressions without any constraints \textbf{QRlin} and the other with convexity constraint \textbf{QRconv}.

We consider the \textbf{Pointwise}, \textbf{B-Constant}, \textbf{P-Constant}, \textbf{B-Smooth} and \textbf{P-Smooth} specifications (see Section~\ref{sim1}). However, the general \textbf{B-Smooth} specifications (see~\eqref{eq_qr_basis_p}) using quantile regression were excluded from the experiment because of their very high computational cost.
The standard CRPS-learning specification \textbf{B-Constant} was used in combination with \textbf{ML-PolyG} in \citet{thorey2018ensemble}. Further, \citet{zamo2021sequential} considered \textbf{EWA} and \textbf{EWAG} using \textbf{B-Constant}. 

For the P-smoothing methods we consider $\alpha=0.5$ and the $\lambda$-grid
$\Lambda= \{0\}\cup \{2^x|x\in \{-4,-3,\ldots,13\}\}\cup \{2^{30}\}$ where the latter value approximates the $\lambda\to \infty$ situation. For the B-spline smoothing methods we consider 19 cubic B-spline basis on equidistant knot distances with distances $2^{\mathcal{K}}$ where $\mathcal{K}$ is a linear grid from $\log_2(.01)$ to $0.5$ of size 19. We augment the 19 basis functions by the constant function.
For EWAG and BMA we specify an $\eta$-grid of learning rates $\mathcal{E}= \{2^x|x\in \{-3,-2.8,\ldots,9\}\}$ additionally. We choose the tuning parameters online based on the past performance of the forecaster as described in Subsection~\ref{subsec:impl}.

Table~\ref{tab_eua} presents the results. It shows the CRPS difference of the considered experts and aggregation methods with respect to the \textbf{Naive} aggregation procedure and aggregated across time. Additionally, the table shows the p-value of the Diebold-Mariano test (DM-test), which tests if the considered forecaster provides better forecasts than \textbf{Naive}.

\begin{table}[h]
      \centering
      \resizebox{1\textwidth}{!}{
            \begin{tabular}{ccccc}
                  \hline
                  \textbf{ETS-ANN}                                                    & \textbf{QuantReg}                                                   & \textbf{ARMA-GARCH}                                      & \textbf{I-EGARCH}                                        & \textbf{I-GARCHt}                                         
                  \\
                  \hline
                  \cellcolor[rgb]{1,0.5,0.55}2.101 {\footnotesize (\textgreater.999)} & \cellcolor[rgb]{1,0.5,0.55}1.358 {\footnotesize (\textgreater.999)} & \cellcolor[rgb]{1,0.682,0.5}0.520 {\footnotesize (.993)} & \cellcolor[rgb]{1,0.688,0.5}0.511 {\footnotesize (.999)} & \cellcolor[rgb]{0.944,1,0.5}-0.037 {\footnotesize (.406)}
                  \\
                  \hline
            \end{tabular}
      }
      \resizebox{\textwidth}{!}{
            \setlength{\tabcolsep}{2pt}
            \begin{tabular}{rcccccc}
                                      & \textbf{BOAG}                                                     & \textbf{EWAG}                                                & \textbf{ML-PolyG}                                            & \textbf{BMA}                                             & \textbf{QRlin}                                                     & \textbf{QRconv}                                          \\
                  \hline
                  \textbf{Pointwise}  & \cellcolor[rgb]{0.541,0.914,0.5}-0.170 {\footnotesize(.055)}      & \cellcolor[rgb]{0.822,1,0.5}-0.089 {\footnotesize(.175)}     & \cellcolor[rgb]{0.645,0.948,0.5}-0.141 {\footnotesize(.112)} & \cellcolor[rgb]{1,0.973,0.5} 0.032 {\footnotesize(.771)} & \cellcolor[rgb]{1,0.5,0.55}3.482 {\footnotesize(\textgreater.999)} & \cellcolor[rgb]{0.989,1,0.5}-0.019 {\footnotesize(.309)} \\
                  \textbf{B-Constant} & \cellcolor[rgb]{0.73,0.977,0.5}-0.118 {\footnotesize(.146)}       & \cellcolor[rgb]{0.916,1,0.5}-0.049 {\footnotesize(.305)}     & \cellcolor[rgb]{0.819,1,0.5}-0.090 {\footnotesize(.218)}     & \cellcolor[rgb]{1,0.969,0.5} 0.038 {\footnotesize(.834)} & \cellcolor[rgb]{1,0.5,0.55}4.002 {\footnotesize(\textgreater.999)} & \cellcolor[rgb]{1,0.671,0.5} 0.539 {\footnotesize(.996)} \\
                  \textbf{P-Constant} & \cellcolor[rgb]{0.658,0.953,0.5}-0.138 {\footnotesize(.020)}      & \cellcolor[rgb]{0.867,1,0.5}-0.070 {\footnotesize(.137)}     & \cellcolor[rgb]{0.675,0.958,0.5}-0.133 {\footnotesize(.026)} & \cellcolor[rgb]{1,0.968,0.5} 0.039 {\footnotesize(.851)} & \cellcolor[rgb]{1,0.5,0.55}5.275 {\footnotesize(\textgreater.999)} & \cellcolor[rgb]{1,0.986,0.5} 0.009 {\footnotesize(.683)} \\
                  \textbf{B-Smooth}   & \cellcolor[rgb]{0.533,0.911,0.5}-0.173 {\footnotesize(.062)}      & \cellcolor[rgb]{0.879,1,0.5}-0.065 {\footnotesize(.276)}     & \cellcolor[rgb]{0.647,0.949,0.5}-0.141 {\footnotesize(.118)} & \cellcolor[rgb]{0.934,1,0.5}-0.042 {\footnotesize(.386)} & -                                                                  & -                                                        \\
                  \textbf{P-Smooth}   & \cellcolor[rgb]{0.5,0.9,0.5}\textbf{-0.182} {\footnotesize(.039)} & \cellcolor[rgb]{0.767,0.989,0.5}-0.107 {\footnotesize(.121)} & \cellcolor[rgb]{0.578,0.926,0.5}-0.160 {\footnotesize(.065)} & \cellcolor[rgb]{1,0.968,0.5} 0.040 {\footnotesize(.804)} & \cellcolor[rgb]{1,0.5,0.55}3.495 {\footnotesize(\textgreater.999)} & \cellcolor[rgb]{1,0.999,0.5}-0.012 {\footnotesize(.369)} \\
                  \hline
            \end{tabular}
      }
      \caption{CRPS difference to \textbf{Naive} (scaled by $10^4$) of single experts, and six combination methods with five weight function specifications. Negative values correspond to better performance (the best value is bold). Additionally, we show the p-value of the DM-test, testing against \textbf{Naive}. The cells are colored with respect to their p-values (the greener better).
      }\label{tab_eua}
\end{table}

We observe that all experts except \textbf{I-GARCHt} perform clearly worse than the \textbf{Naive}. When looking at the aggregation methods, we see that the online learning methods \textbf{BOAG}, \textbf{EWAG} and \textbf{ML-PolyG} tend to outperform \textbf{Naive}.
We receive the best results using \textbf{BOAG}. Moreover, \textbf{P-Smooth} provides better results than \textbf{Pointwise} for the three online learners. In contrast, \textbf{B-Smooth} fails to do so.
Still, we see that the lowest p-value is obtained for \textbf{BOAG} using \textbf{P-Constant}. The reason is that pointwise and smoothing procedures lead to more volatile predictions.
Furthermore, we observe that the quantile regression methods are hardly worth considering, especially regarding the high computational effort. Therefore we exclude them from the proceeding graphs and focus on online methods.

Figures~\ref{fig:eua_ql} and~\ref{fig:eua_loss_time} present the main results when evaluating forecasting performance with respect to \textbf{Naive} across probabilities and time.
From the quantile loss differences in Figure~\ref{fig:eua_ql} we see that the experts' performance varies clearly across the probability, as mentioned a key sign that pointwise CRPS learning should be considered. The \textbf{ETS-ANN} performs fine in the center and poorly otherwise. The \textbf{I-GARCHt} performs very well in the lower part of the distribution but in the upper part not clearly better than other experts. Our benchmark, the \textbf{Naive} combination, performs better than all experts in the upper part of the distribution and worse than \textbf{I-GARCHt} in the lower part. The \textbf{BMA} performs similarly. The other learning methods
\textbf{BOAG}, \textbf{EWAG} and \textbf{ML-PolyG} exhibit similar pattern, in the lower part of the distribution they perform similarly to best single expert \textbf{I-GARCHt}, and in the upper part similarly to \textbf{Naive}.

\begin{figure}[h]
      \centering
      \fbox{\includegraphics[width=0.975\columnwidth]{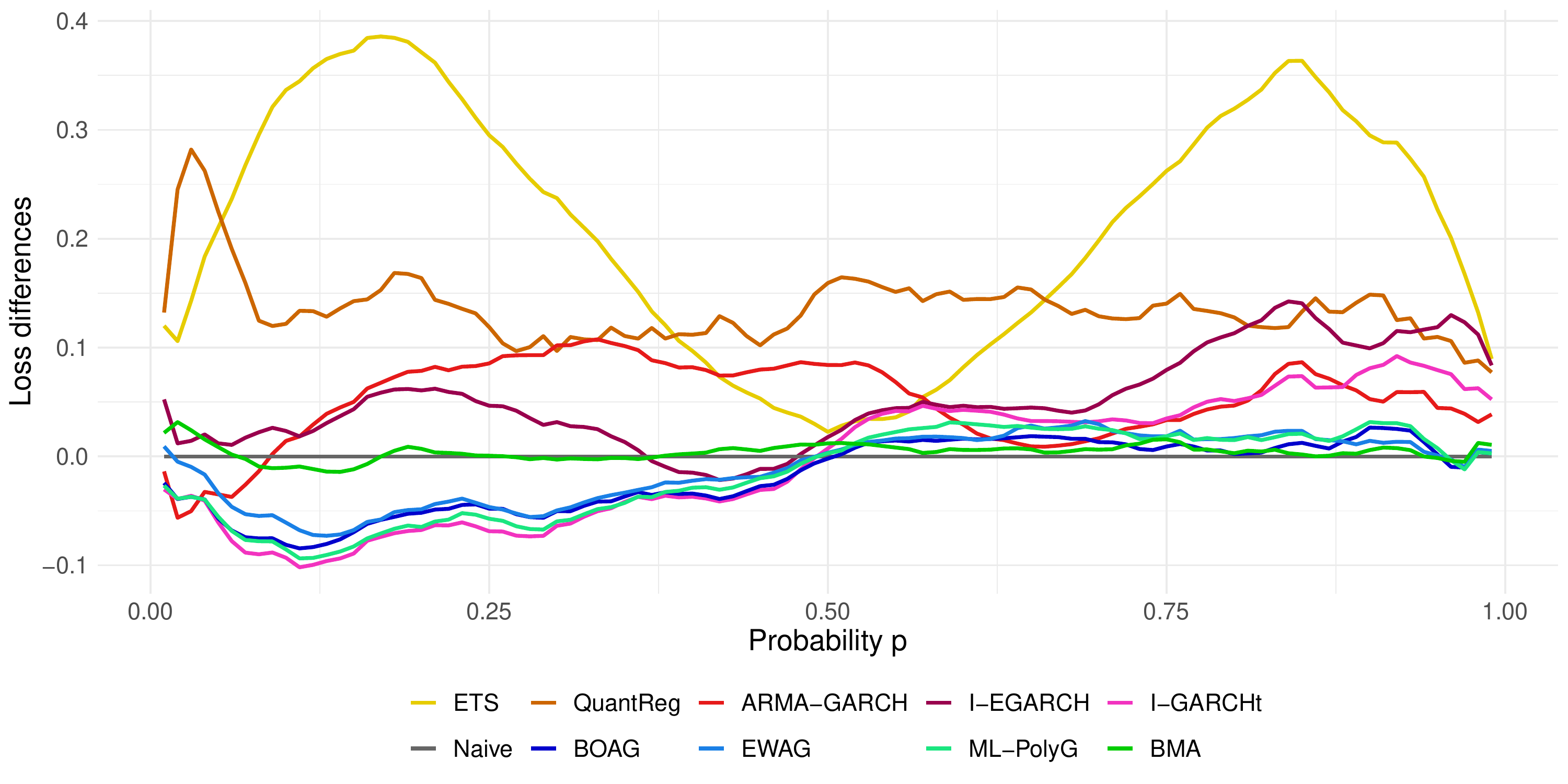}}
      \caption{Quantile loss difference to \textbf{Naive} (scaled by $10^4$) of single experts, and four combination methods with smoothing procedure \textbf{P-Smooth}. Smaller values correspond to better performance.}\label{fig:eua_ql}
\end{figure}

\begin{figure}[h]
      \centering
      \fbox{\includegraphics[width=0.975\columnwidth]{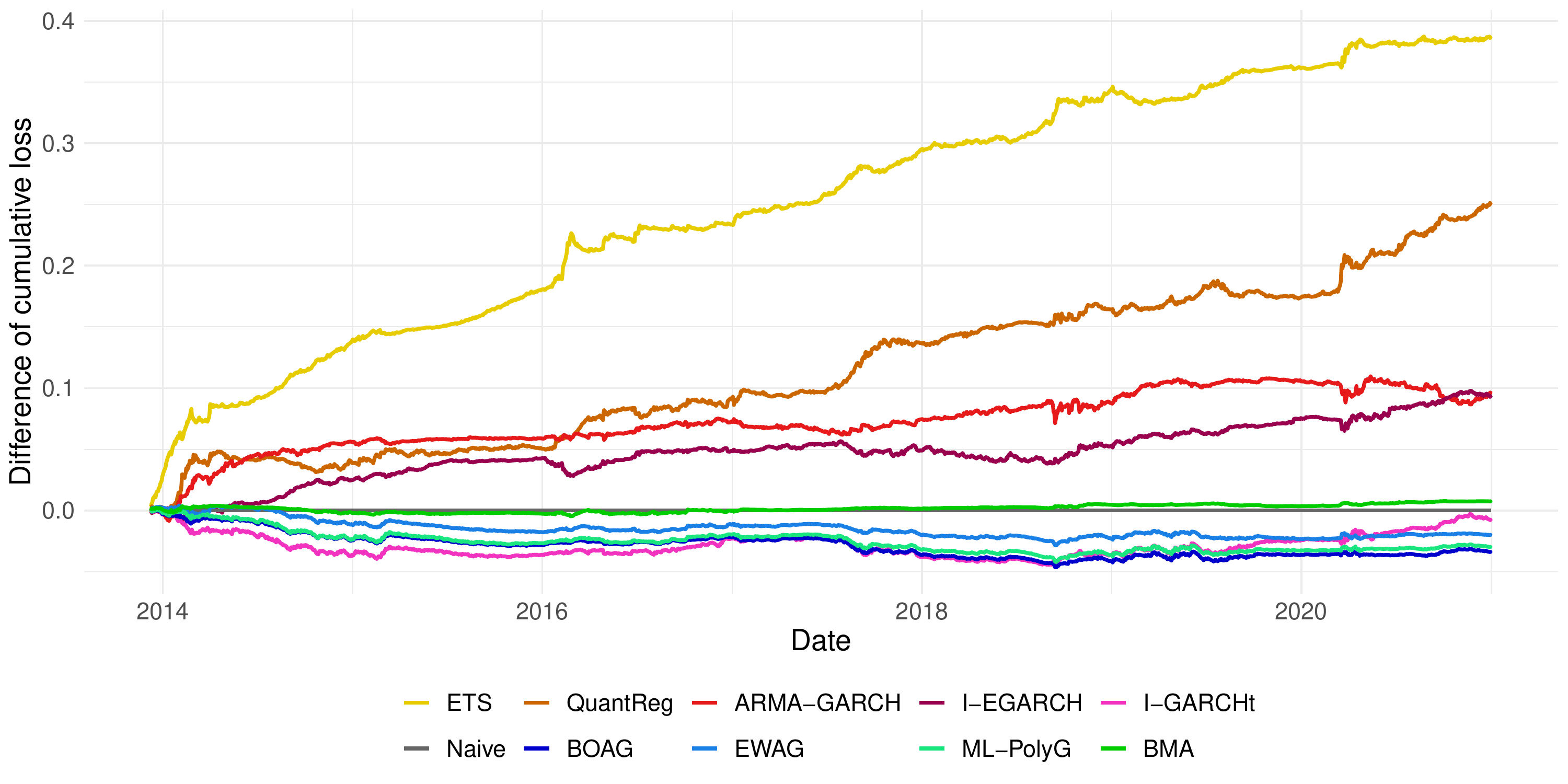}}
      \caption{Cumulative CRPS difference to \textbf{Naive} of single experts, and four combination methods with smoothing procedure \textbf{P-Smooth}.}
      \label{fig:eua_loss_time}
\end{figure}

From the cumulated CRPS differences in Figure~\ref{fig:eua_loss_time} we see that the overall results also hold over time. The likely most interesting observation is that in the first year, \textbf{I-GARCHt} performs better than the combination methods. Thereafter \textbf{BOAG}, \textbf{EWAG} and \textbf{ML-PolyG} catch up and eventually overtake the \textbf{I-GARCHt}. The reason may be that the learners require some time to stabilize. We do not consider  3- or 4-parametric distributions to limit the computation effort.

Figure~\eqref{fig:eua_weights_last} presents the weight functions in the last time step for the five specifications of the BOAG, which is the best online learning method. We see that for \textbf{Pointwise} the weights vary substantially over $\bsPP$. Also, it is noteworthy that the weights of both constant solutions differ to a moderate extent. Furthermore, we observe that \textbf{P-Smooth} leads to more flexible (i.e., locally smoothed) weights than the \textbf{B-Smooth}, which applies a global smoothing.
\begin{figure}[htbp]
      \centering
      \fbox{\includegraphics[width=0.975\columnwidth]{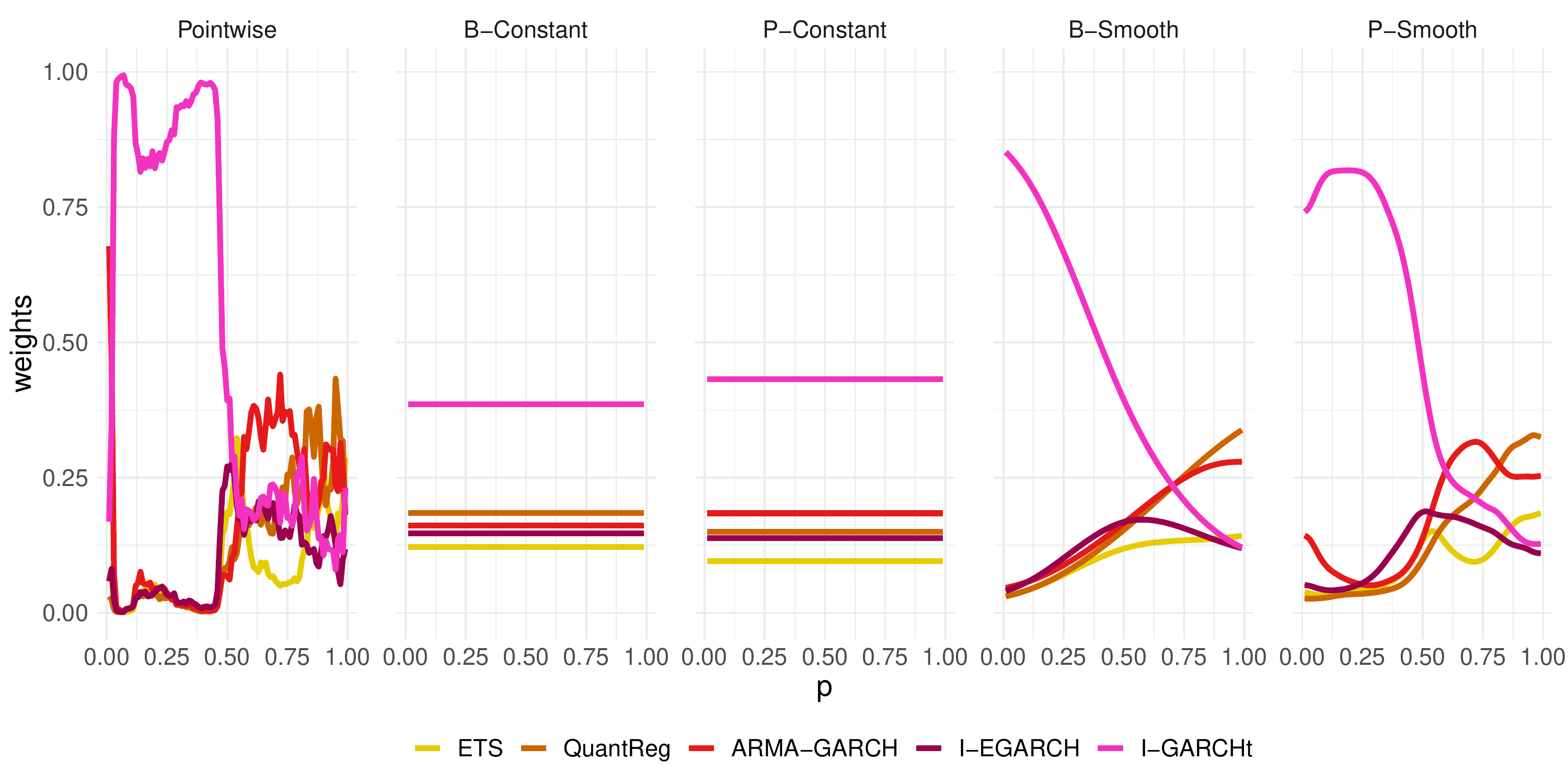}}
      \caption{Weight functions of \textbf{BOAG} in the last update step with respect to all discussed specifications.}\label{fig:eua_weights_last}
\end{figure}

Figure~\ref{fig:eua_weights} relates to the right tile of Figure~\eqref{fig:eua_weights_last} as it illustrates the evolution of the weight functions over time for \textbf{BOAG} with the {P-Smooth} specification. Indeed, we observe that the weights change rapidly in the initial period. Overall, we see that the performance pattern of~\eqref{fig:eua_ql} translates to the weight functions. The \textbf{I-GARCHt} dominated the weight in the lower part of the distribution. Only in the extreme tails, the \textbf{ARMA-GARCH} contributes all the time. For the first two years, the quantile regression \textbf{QuantReg} also contributes to the weight function in the lower part of the distribution. However, the weights in the upper part of the distribution are much closer to the \textbf{Naive} combination.

\begin{figure}[htbp]
      \centering
      \fbox{\includegraphics[width=0.975\columnwidth]{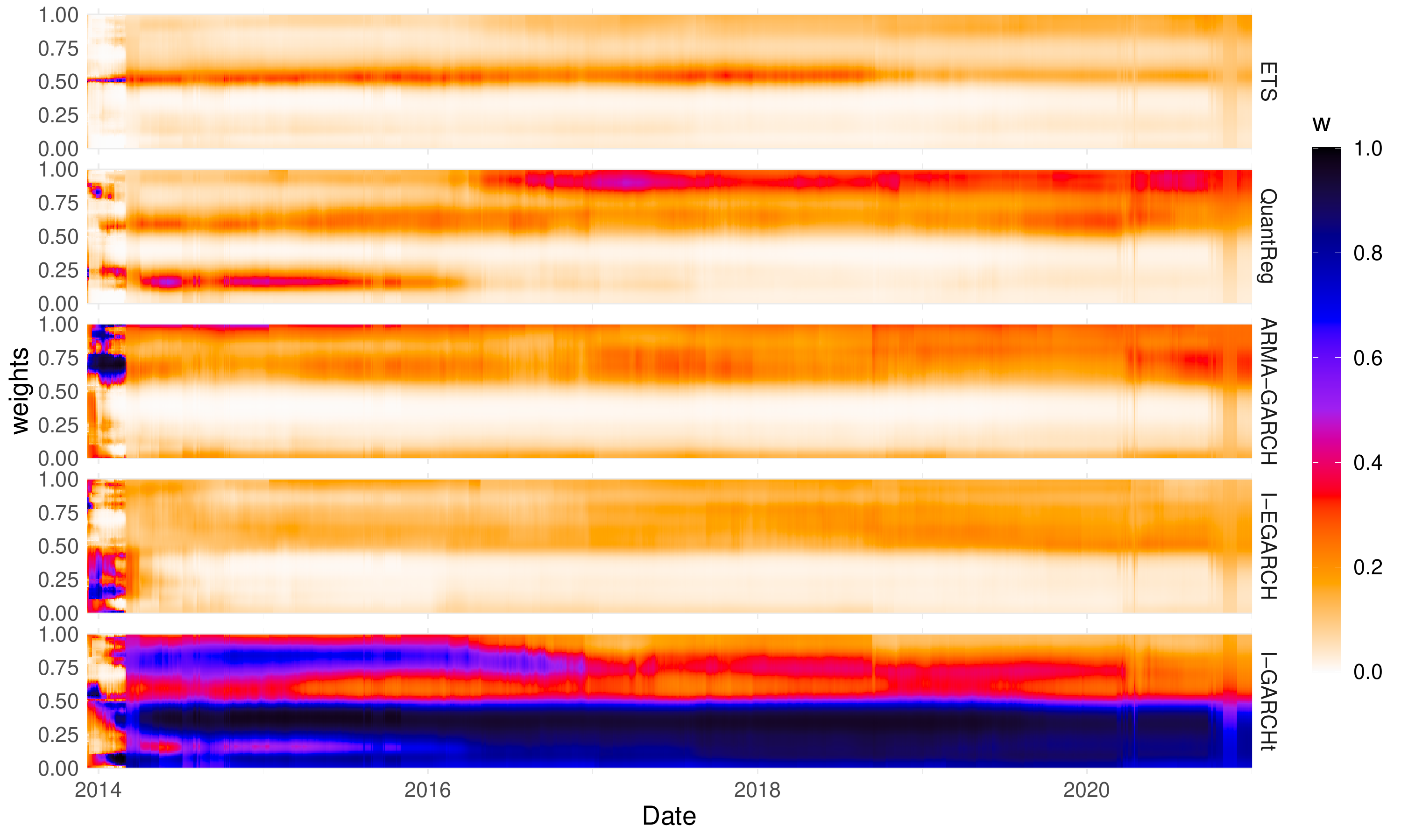}}
      \caption{Weight functions over time of \textbf{BOAG P-Smooth} for all experts.}\label{fig:eua_weights}
\end{figure}

%% file: 06_disc_concl.tex
This paper presented a framework for probabilistic forecast combination that aims to minimize the CRPS.

The key idea is to aggregate probabilistic expert predictions with combination weights that vary in time and parts of the distribution. The time-adaptivity is relevant to account for the changing performance of the experts. The variation of the weights in the distribution is important to consider that, e.g., some experts perform better in the center of the distribution while others perform better in the distributions' tails.
We prove that we can expect improvements in performance in those situations by considering pointwise optimization of the weight functions (see Proposition~\ref{proposition_optimality}).

We study the resulting CRPS learning approach for batch and online learning methods. The combination weight functions are modeled by a linear combination of basis functions, where we suggest B-Splines. We show that quantile regression can be used for CRPS learning but can be computationally costly. In contrast, we show that online aggregation is a suitable solution in practice. We prove that under mild regularization constraints, the fully adaptive BOA can yield almost\footnote{There is a small lack of optimality due to an $\log(\log(t))$ term which results from the adaptive learning rates.} optimal convergence rates with respect to the selection problem and the convex aggregation problem (see Theorem~\ref{thm_boag}). In simple words,
the proposed BOA for CRPS learning guarantees fast convergence $\OO(t)$ to the best expert and standard convergence $\OO(\sqrt{t})$ to the best convex combination of all experts. To achieve optimal convergence rates, we require some convexity assumptions on the quantile risk of the predictive target distribution. Proposition~\ref{proposition_qrisk} helps to characterize those convexity properties.

Furthermore, we discuss relationships to BMA and extensions to the proposed B-Spline-based CRPS learning method.
Most notably, we discuss P-Spline smoothing and propose a corresponding fully adaptive BOA procedure for CRPS-learning.
This online learning algorithm yields preferable probabilistic forecasting performance in simulation studies and an application to probabilistic forecasting of European emission allowance prices.

Multiple slightly different variations of the BOA exist in the literature~\citep{wintenberger2017optimal, bregere2020online}. The differences concern the cumulative regret, the weights, and the fact that $E$ is forced to be an integer exponentiation with base 2 in \citet{wintenberger2017optimal} while the integer condition is relaxed in \citet{bregere2020online}.
The \texttt{profoc} \texttt{R}-Package~\citet{profoc_package} implements relevant BOA versions, various related online learning procedures (e.g. EWA, ML-Poly) and quantile regression-based methods for CRPS-learning. Additionally, B-spline and P-spline smoothing can be applied to all algorithms.

Further research could go in various directions.
We restrict ourselves to equidistant quantile grids, as it is a popular reporting option for probabilistic forecasts. However, it only \textit{approximates} the CRPS. Precise CRPS learning methods for popular parametric distributions could be studied in more depth. It seems plausible that faster algorithms can be achieved, as we would not have to evaluate a dense grid of quantiles.

Moreover, we considered pointwise aggregation across quantiles (see~\eqref{eq_forecast_F_def}) for CRPS learning. Alternatively, we could study pointwise aggregation across probabilities by considering
\begin{equation}
    \widetilde{F}_{t}(x) = \sum_{k=1}^K w_{t,k}(x) \widehat{F}_{t,k}(x) .
    \label{eq_av_across_prop}
\end{equation}
This is more popular in the density combination literature~\citep{aastveit2014nowcasting, kapetanios2015generalised} but, to our knowledge, not studied for estimating optimal weight functions.
In this CRPS learning the weighting scheme~\eqref{eq_av_across_prop}
could utilize the standard definition of the CRPS~\eqref{eq_crps}.
However, the unbounded support of the weight functions in~\eqref{eq_av_across_prop} creates a more complicated CRPS learning problem. The main reason is that suitable consideration of basis functions should depend on the distribution of the prediction target, esp. appropriate range, location, and scale parameters. However, it might be worth investigating it further. Another interesting path of research may be the consideration of pointwise combination methods for distributional forecasts with respect to other scoring rules, e.g., the log-score.

Finally, further investigations could go into the direction of online aggregation methods that utilize the expert predictions and simultaneously the in-sample (i.e., training) performance of the experts (e.g., in BMA). Such a procedure is studied  in~\citet{adjakossa2020kalman} for the $\ell_2$-loss.

%% file: 99_appendix.tex
\begin{proof}[Proof of Proposition~\ref{proposition_optimality}]\label{proof_proposition_optimality}
    We show~\eqref{eq_risk_ql_crps_convex}.
    Let $\widetilde{F}^{-1}_{t,\conv}(p) = \sum_{k=1}^K w_{t,k}(p) \widehat{F}^{-1}_{t}(p)
    $ the pointwise optimal solution with respect to $\QL_p$
    and $(\widetilde{F}^*)^{-1}_{t,\conv} =
        \sum_{k=1}^K w_{t,k}^* \widehat{F}^{-1}_{t} $
    the optimal solution with respect to $\text{CRPS}$
    for $t$.
    With Fubini and~\eqref{eq_crps_qs}
    it holds
    \begin{align*}
         & \widehat{\mathcal{R}}^{\text{CRPS}}_{t,\conv} - \overline{\widehat{\mathcal{R}}}^{\QL}_{t,\conv}                                                                                                    \\
         & =
        \sum_{i=1}^t \mathbb{E}[\text{CRPS}(\widetilde{F}^*_{i,\conv}, Y_i )|\mathcal{F}_{i-1} ] - 2\int_{0}^1 \sum_{i=1}^t \mathbb{E}[{\QL}_p(\widetilde{F}^{-1}_{i,\conv}(p), Y_i )|\mathcal{F}_{i-1} ] \,dp \\
         & =
        2 \int_{0}^1 \underbrace{\sum_{i=1}^t \mathbb{E}\left[ {\QL}_p((\widetilde{F}_{i,\conv}^*)^{-1}(p), Y_i ) |\mathcal{F}_{i-1} \right] - \sum_{i=1}^t 
            \mathbb{E}\left[
                {\QL}_p(\widetilde{F}^{-1}_{i,\conv}(p), Y_i ) |\mathcal{F}_{i-1} \right]}_{\geq 0 \text{ by optimality of } \widetilde{F}_{t,\conv} \text{ for all }p\in(0,1) } \,dp .
    \end{align*}
    The inequality equality holds strictly if and only if the optimality of $\wtilde{F}^{-1}_{i,\text{conv}}$ is strict on any interval $(p-\epsilon,p)$ due to the continuity of ${\QL}_p$ and the fact that ${\QL}_p(q,y)>0$ for $y-q\neq 0$. This holds, if and only if the weights $w_{t,k}$ are not constant for at least one $k$ in $p$ on the corresponding interval as $\wtilde{F}^{-1}$ is one-sided continuous.

    \eqref{eq_risk_ql_crps_expert} and~\eqref{eq_risk_ql_crps_lin} follow the same argument. In the former case, considering
    $\widetilde{F}^{-1}_{t,\min}(p) = \min_{k=1,\ldots,K} \widetilde{F}^{-1}_{t}(p)
    $ the pointwise individual best expert with respect to $\QL_p$
    and $(\widetilde{F}^*)^{-1}_{t,\min} =
        \min_{k=1,\ldots,K} \widetilde{F}^{-1}_{t} $
    the optimal solution with respect to $\text{CRPS}$.
\end{proof}

\begin{proof}[Proof of Proposition~\ref{proposition_qrisk}]\label{proof_proposition}
    We observe that $ \mathcal{Q}_p(x) =
        \int_{\R}  (\mathbb{1}\{0 < x-y\} -p)(x - y) f(y) d\nu(y)
        = \int_{\R}  \rho_p(x-y) f(y) d\nu(y) = \rho_p * f$
    with 
    $*$ as convolution operator. Then we receive by exchangability of subgradients in covolutions for the second subgradient that
    $(\rho_p * f)'' = (\rho_p' * f)' = \rho_p'' * f$. Further,
    we have that $\rho_p'(x)= (p-\boldsymbol{1}\{x\leq 0 \} ) $ and $\rho_p'' = \delta_0$ with $\delta_0$ as dirac function in $0$. Thus, the result follows by $\rho_p'' * f = \delta_0 * f = f$. 
    The latter statement of the proposition concerning follows as a twice-differentiable real valued function $g$ is $\gamma$-strongly convex if and only if
    $g''>\gamma >0$.
\end{proof}

\begin{proof}[Proof of Theorem~\ref{thm_boag}]\label{proof_theorem}
    As we want to apply Theorem 2.1 of \citet{gaillard2018efficient} we remark that a loss $\ell(\cdot, Y_t)$ is convex $G$-Lipschitz and weak exp-concave if
    \begin{itemize}
        \item[(A1)]
            for some $G>0$ it holds
            for all $x_1,x_2\in \R$ and $t>0$ that
            $$ | \ell(x_1, Y_t)-\ell(x_2, Y_t) | \leq G |x_1-x_2|$$
        \item[(A2)] for some $\alpha>0$, $\beta\in[0,1]$ it holds
            for all $x_1,x_2 \in \R$ and $t>0$ that
            \begin{align*}
                \mathbb{E}[
                    \ell(x_1, Y_t)-\ell(x_2, Y_t)
                | \mathcal{F}_{t-1}] \leq &
                \mathbb{E}[ \ell'(x_1, Y_t)(x_1 -  x_2)  |\mathcal{F}_{t-1}] \\
                                          & +
                \mathbb{E}\left[ \left. \left( \alpha(\ell'(x_1, Y_t)(x_1 -  x_2))^{2}\right)^{1/\beta}  \right|\mathcal{F}_{t-1}\right].
            \end{align*}
    \end{itemize}

    As $\QL_p$ is convex Theorem 4.3 of \citet{wintenberger2017optimal} implies~\eqref{eq_boa_opt_conv} for all $p$.
    Similarly, if $Y_t|\mathcal{F}_{t-1}$ is bounded and has a pdf $f$ Proposition 1 implies that quantile risk $\mathcal{Q}_p(x) = \mathbb{E}[ \QL_p(x, Y_t|\mathcal{F}_{t-1}) ]$ at $t$ is $\gamma$-strongly convex and satisfies (A2).
    Furthermore, from~\eqref{ql_gradient} we see that
    $\QL_p$ is convex Lipschitz with optimal Lipschitz constant $
        c_p=\max\{ p, 1-p\} $.
    Thus, by Theorem 2.1 of \citet{gaillard2018efficient} we receive~\eqref{eq_boa_opt_select} with $\beta=1$.
    Now, define $\bsPP_j = \{ a 2^{-j} | a=1,\ldots, 2^{a}-1\}$,
    the CRPS-approximation $\text{CRPS}_j = \frac{1}{2^{j}-1}\sum_{p\in \bsPP_j} \QL_{p}$ and
    note that $\text{CRPS} = \lim_{j\to\infty} \text{CRPS}_j$.
    $\text{CRPS}_j$ and thus $\text{CRPS}$ satisfies (A2) with $\beta=1$ as $\gamma$ does not depend on $p$. Further, with some calculus we observe that $\text{CRPS}_j$ has the optimal Lipschitz constant
    $C_j = \sum_{l=1}^{2^j-1} \max\{l,1-l\} /(2^j-1) =
        (3(2^{j-1}-1)  + 1)/(2 (2^j-1)) $.
    Hence, for $C_j$ it holds $\lim_{j\to\infty }(C_j) = 3/4$.
    Thus, $\text{CRPS}$ is convex Lipschitz
    and satiesfies (A1) and the Theorem is implied by Theorem 2.1 of \citet{gaillard2018efficient}.
\end{proof}